\documentclass[11pt]{article}
\usepackage{fullpage}
\usepackage{dsfont}
\usepackage{amsmath}
\usepackage{amssymb}
\usepackage{mathtools}
\usepackage{amsthm}
\usepackage{amsfonts}
\usepackage{bm}
\usepackage{graphicx}
\usepackage{enumerate}
\usepackage{url} % not crucial - just used below for the URL 
\usepackage{xurl} % If the URL is too long, this package will make it Line-Breakable.

\usepackage{hyperref}
\usepackage[round]{natbib}

\usepackage{microtype}
\usepackage{subfigure}
\usepackage{booktabs} % for professional tables
\usepackage{multirow}
\usepackage{array}
\usepackage{caption}
\usepackage{rotating}
\usepackage{xcolor}
\usepackage{algorithm}
\usepackage{algorithmic}
\usepackage{setspace}

\usepackage{amsmath}
\usepackage[bb=dsserif]{mathalpha}
\usepackage{amssymb}
\usepackage{mathtools}
\usepackage{amsthm}
\usepackage{amsfonts}
\usepackage{bm}
\usepackage{graphicx}
\usepackage{enumerate}
\usepackage{url}
\usepackage{bm}
\usepackage{comment}
\usepackage{microtype}
\usepackage{subcaption}
\usepackage{booktabs} % for professional tables
\usepackage{multirow}
\usepackage{array}
\usepackage{caption}
\usepackage{rotating}
\usepackage{xcolor}
\usepackage{algorithm}
\usepackage{algorithmic}
\usepackage{setspace}

\newcommand{\cM}{{\mathcal M}}

\newcommand{\cP}{{\mathcal P}}
\newcommand{\cW}{{\mathcal W}}
\newcommand{\cS}{{\mathcal S}}
\newcommand{\cR}{{\mathcal R}}
\newcommand{\cG}{{\mathcal G}}

\newcommand{\bP}{\mathbb{P}}
\DeclareMathOperator{\argmax}{argmax}
\DeclareMathOperator{\iid}{i.i.d.}

\DeclareMathOperator{\Var}{Var}

\theoremstyle{plain}
\newtheorem{theorem}{Theorem}[section]
\newtheorem{proposition}[theorem]{Proposition}
\newtheorem{lemma}[theorem]{Lemma}

\theoremstyle{definition}
\newtheorem{definition}[theorem]{Definition}

\theoremstyle{remark}
\newtheorem{remark}[theorem]{Remark}

\hypersetup{
colorlinks=true,
filecolor=black,
citecolor=blue,
urlcolor=black,
}

\title{Debiasing Watermarks for Large Language Models via Maximal Coupling}

\date{} 

\author{Yangxinyu Xie\thanks{University of Pennsylvania}
    \\
    \and
    Xiang Li$^*$\\
    \and
    Tanwi Mallick\thanks{Argonne National Laboratory}\\
    \and
    Weijie J.~Su$^*$\\
    \and
    Ruixun Zhang\thanks{Peking University}    
}

\hypersetup{
pdftitle={Debiasing Watermarks for Large Language Models via Maximal Coupling},
pdfsubject={Large Language Models, Watermarking, Maximal Coupling},
pdfauthor={Yangxinyu Xie, Xiang Li, Tanwi Mallick, Weijie J.~Su, Ruixun Zhang},
pdfkeywords={Large Language Models, Watermarking, Maximal Coupling}
}

\begin{document}

\maketitle

\begin{abstract}
    Watermarking language models is essential for distinguishing between human and machine-generated text and thus maintaining the integrity and trustworthiness of digital communication. We present a novel green/red list watermarking approach that partitions the token set into ``green'' and ``red'' lists, subtly increasing the generation probability for green tokens. To correct token distribution bias, our method employs maximal coupling, using a uniform coin flip to decide whether to apply bias correction, with the result embedded as a pseudorandom watermark signal. Theoretical analysis confirms this approach's unbiased nature and robust detection capabilities. Experimental results show that it outperforms prior techniques by preserving text quality while maintaining high detectability, and it demonstrates resilience to targeted modifications aimed at improving text quality. This research provides a promising watermarking solution for language models, balancing effective detection with minimal impact on text quality.
\end{abstract}

\section{Introduction}
\label{sec:introduction}

The field of artificial intelligence (AI) has seen rapid growth in recent years, with large language models (LLMs) playing a central role in many applications. These models, such as OpenAI's ChatGPT \citep{achiam2023gpt}, are capable of understanding and generating human-like text. However, the potential for misuse has raised concerns about the trustworthiness of AI-generated content, particularly in the context of disinformation, fraud, social engineering, and cybercrime \citep{Europol2023}. To mitigate these risks, President Biden issued an executive order on Safe, Secure, and Trustworthy Artificial Intelligence, which emphasized the need for standards to detect AI-generated content and authenticate official content, setting a precedent for transparency and trust in digital communications \citep{Biden2023AIOrder}. In particular, the Department of Commerce identified watermarking as a pivotal means to clearly label AI-generated content. 

Watermarks embed detectable statistical signals into the text generation process with minimal impact on text quality. 
We introduce a watermarked text generation process based on the framework established by~\citet{li2024statistical}. 
Let $\cW$ denote the vocabulary of tokens.
When inputted with a sequence of tokens $w_{1:t-1} = (w_1, ..., w_{t-1})$,\footnote{For the purpose of our discussion, it is helpful to think of $w_{1:t-1}$ as a sequence of words, but in practice, these tokens can be words, parts of words, or punctuation marks. For example, the sentence ``Hello, world!'' when tokenized, might be split into four tokens: [``Hello'', ``,'', ``world'', ``!''].} the language model first generates the next-token-prediction (NTP) distribution ${\bm P}_t=(P_{t,w})_{w \in \cW}$ satisfying $\sum_{w \in \cW} P_{t,w} = 1$, and then samples the next token $w_t$ according to ${\bm P}_t$.
The sampling step is realized by a decoder $\cS$ so that $w_t := \cS({\bm P}_t, \xi_t)$ is a (possibly stochastic) function of ${\bm P}_t$ and a pseudorandom variable $\xi_t$.
After generating this next token, the model appends it to the sequence of tokens, obtains $w_{1:t} = (w_1, ..., w_{t})$, and repeats the process until a stop token is generated or a predefined maximum length is reached.\footnote{For detailed explanations with clear visuals, see \citep{HuggingFace2024GenerationLLMs}.}
The watermark signal embedded in the decoder $\cS$ essentially lies in the dependence of each token $w_t$ on the pseudorandom variable $\xi_t$, where $\xi_t$ provides the additional information needed for watermark detection. Here, pseudorandomness implies that $\xi_{1:n}$ can be reconstructed if the generated text $w_{1:n}$ (along with some additional shared information) is available, and this collection of shared information is referred to as the watermark key.

The earliest watermarking scheme, the Gumbel-max watermark, was introduced by \citet{aaronson}. This method selects the next token $w_t$ using the decoder $w_t = \cS(\bm{P}_t, \xi_t) := \argmax_{w \in \cW} \log \frac{U_w}{P_w}$, assuming $\xi_t = (U_1, \ldots, U_{|\cW|})$ is a random vector consisting of $|\cW|$ i.i.d. copies of $U[0,1]$, serving as part of the watermark key. While this watermarked decoder theoretically preserves the original next-token distribution when $\xi_t$ is truly random and achieves high detection efficiency, this decoder can generate repetitive text and significantly degrade text quality in practice \citep{kuditipudi2023robust}. The main reason is that, for existing implementations, the pseudorandom variable $\xi_t$ depends deterministically on the prior $k$ tokens $w_{(t-k)}, \ldots, w_{(t-1)}$. As a result, when conditioned on previous tokens, the decoder $\cS$ deterministically outputs the next token, which can lead to sub-optimal text generation \citep{holtzmancurious}.

Another influential watermarking scheme, the green/red list watermark, employs a stochastic decoder. 
First introduced by \citet{kirchenbauer2023watermark}, this approach has gained popularity \citep{fernandez2023three, wu2023dipmark}. It first partitions the token vocabulary $\cW$ into the green list $\cG$ and the red list $\cR = \cW \setminus \cG$ (pseudo)randomly, and shares the green list $\cG$ with the detector as part of the watermark key. The watermark is embedded by prompting the use of green tokens during decoding: given the NTP distribution ${\bm P}_t$ and the green list $\cG$, the watermarked decoder constructs an alternative distribution ${\bm Q}_t = (Q_{t,w})_{w \in \cW}$ and samples the next token from ${\bm Q}_t$. The simplest version is the ``hard'' watermark, where each sampled token is limited in the green list $\cG$:
\begin{equation}
    Q_{t,w} = 
    \frac{P_{t,w}}{P_{t,\cG}} \mathbb{1}_{\{w \in \cG\}},
    \label{eq:hard green list}
\end{equation}
where $P_{t,\cG} := \sum_{w' \in \cG}P_{t,w'}$ is the normalization constant. The intuition is that if a token $w_t$ is in the green list $\cG$, then it is more likely from ${\bm Q}_t$ than ${\bm P}_t$; thus, one can detect this difference by counting the number of green tokens in a given text.\footnote{In practice, the green lists would also be pseudorandom variables of the previous $k$ tokens.}
For instance, if each green list contains half of the vocabulary, a human writer would hit green tokens about half the time. By contrast, a watermarked model deliberately increases the frequency of green tokens, producing text composed solely of green tokens. However, limiting sampling to a vocabulary subset distorts the NTP distribution $\bm P_t$ and can degrade LLM generation significantly \citep{kirchenbauer2023watermark}. 
Despite these potential drawbacks, this watermarked decoder remains stochastic, as additional randomness is introduced during sampling from ${\bm Q}_t$, while pseudorandomness is only used to set the green list. This extra randomness in the sampling process helps avoid repetitive generation.

Each of the two watermarking schemes has its own strengths. On one hand, like the Gumbel-max watermark, we want the decoder $\mathcal{S}$ to faithfully follows the original NTP distributions—a property known as unbiasedness \citep{li2024statistical}. On the other hand, similar to the green/red list watermark, we seek a stochastic decoder $\mathcal{S}$ alleviate the negative impact of the pseudorandom variables on text generation. 
Can we achieve the best of both worlds? 
In this paper, we answer this question affirmatively by introducing a novel watermarking scheme that leverages maximal coupling to balance these goals. The process begins by drawing a standard uniform random variable $\zeta$ and compute $p = \sum_{w \in \cW} \min(P_{t, w},Q_{t, w})$, which represents the overlap between distributions $\bm{P}_t$ and $\bm{Q}_t$. If $\zeta < p$, we sample from the overlapping region $\propto \min(\bm P_t, \bm Q_t)$; otherwise, we sample from the excess distribution $\propto \max(0,\bm{P}_t-\bm{Q}_t)$.\footnote{Both $\min(\bm P, \bm Q)$ and $\max(0,\bm{P}_t-\bm{Q}_t)$ are element-wise operations.} This approach yields an inherently unbiased decoder: for any NTP distributions $\bm P_t$ and $\bm Q_t$, tokens sampled in this manner follow $\bm P_t$. 

To aid detection, we share the information $\xi_t = (\cG, \zeta_t)$, the green list $\cG$ and the random variable $\zeta_t$, with the detector. This scheme benefits from the following property of $\zeta_t$ as a result of the maximal coupling process, and the proof is provided in Section \ref{sec: decoding}:

\begin{lemma}
\label{lem:maximal_coupling}
    For any NTP distribution $\bm P_t$ and green list $\cG$, the conditional distribution of $\zeta_t$ given $\bm P_t$ and the event $\{w_t \in \cG\}$ is uniform on $\left[0, P_{t,\cG}\right]$, where $P_{t,\cG} := \sum_{w' \in \cG}P_{t,w'}$.
\end{lemma}

We utilize the conditional distribution of $\zeta_t$ given $({\bm P}_t, \cG)$ to detect the existence of watermarks via a hypothesis testing framework introduced by \citet{li2024statistical}, which lays two fundamental working hypotheses.
Under the null hypothesis $H_0$, this framework assumes that a human writes the text without the shared watermark key $(\cG, \zeta_t)$. The computed $\zeta_t$ doesn't correlate with the generation of the given text so the identified green tokens should be independent of each $\zeta_{t}$, i.e., each $\zeta_t$ acts as an i.i.d. copy from $U[0, 1]$, regardless of $\bm P_t$ and $\cG$.
Under the alternative hypothesis $H_1$, the text is assumed to be generated using the specified watermarked decoder with the pseudorandom variables $\zeta_{1:n}$
. Consequently, $\zeta_t$ has a known conditional distribution, as described in Lemma \ref{lem:maximal_coupling}. This reduces the watermark detection problem to identifying distributional differences in the values of $\zeta_t$.
\begin{enumerate}
    \item[$H_0$:] the text $w_{1:n}$ is written by human; i.e., $\zeta_t|{(\bm P_t, \cG)} \overset{\iid}{\sim} U[0,1]$ for $t = 1, ..., n$;
    \item[$H_1$:] the text $w_{1:n}$ is written by a language model; i.e., $\zeta_t|{(\bm P_t, \cG)} \sim U\left[0,P_{t,\cG}\right]$ for $t = 1, ..., n$.
\end{enumerate}

However, this alternative hypothesis $H_1$ fails to capture the full complexity of how users interact with the language model: in practice, users may add, delete, or substitute tokens, and subsequently decrease the watermark signal by removing some dependence of $w_t$ on $\zeta_t$.
To address this limitation, we introduce a more general alternative hypothesis $H_1^{(\mathrm{mix})}$ where only a small fraction of the tokens are modified by the user:
\begin{enumerate}
    \item[$H_1^{(\mathrm{mix})}$:] the text $w_{1:n}$ is first generated by a language model and then modified; i.e., $\zeta_t|{(\bm P_t, \cG)} \sim (1 - \varepsilon_n)U[0,1] + \varepsilon_nU\left[0,P_{t,\cG}\right]$ for $t = 1, ..., n$.
\end{enumerate}
Here, $\varepsilon_n$ represents the fraction of non-null effects. 
When $\varepsilon_n\equiv 1$, $H_1^{(\mathrm{mix})}$ is reduced to $H_1$.
Intuitively, if $\varepsilon_n$ and $1-P_{t,\cG}$ are too small, no one can detect the problem $H_0$ versus $H_1^{(\mathrm{mix})}$ because they merge asymptotically.
This formulation is reminiscent of the sparse detection problem \citep{donoho2004higher,donoho2015higher}. 
Following this approach, we focus on the sparse watermark signal case to facilitate theoretical analysis,  where both $\varepsilon_n \asymp n^{-p}$ and $1-P_{t,\cG} \asymp n^{-q}$ decay with $n$ at different rates. In this setting, we explore detectability and feasible detection methods for different regions of $(p, q) \in (0, 1]^2$. The main theoretical result is discussed in detail in Section \ref{sec: detection}. Informally, we will show the following:

\begin{theorem}[Informal]
\begin{enumerate}
    \item Without sufficient watermark signals, i.e., $p+2q > 1$, $H_0$ and $H_1^{(\mathrm{mix})}$ merge asymptotically. Hence, no test can reliably separate $H_0$ and $H_1^{(\mathrm{mix})}$.
    \item Given sufficient watermark signals, i.e., $p+2q < 1$, $H_0$ and $H_1^{(\mathrm{mix})}$ separate asymptotically. In this case, the alternative hypothesis can be reliably detected using the likelihood ratio test and higher criticism proposed by \citet{donoho2004higher,donoho2015higher}.
    \item With slightly stronger watermark signals, i.e., $p+q <0.5$, the sum test, which rejects $H_0$ if the sum of the $\zeta_t$'s is too large, can reliably separate $H_0$ and $H_1^{(\mathrm{mix})}$ asymptotically.
\end{enumerate}
\end{theorem}

In practice, due to proprietary restrictions, the language model provider may not share the NTP distribution ${\bm P}_t$ with the detector. In this case, the likelihood ratio test is not directly applicable. Interestingly, we find that higher criticism (HC), a non-parametric procedure \citep{donoho2004higher,donoho2015higher}, successfully achieves the optimal detection boundary $p+2q=1$ without any knowledge of $\bm P_t$'s and $\varepsilon_n$. Moreover, the detection boundary for the sum test is $p+q =0.5$, leading to less power than HC asymptotically in this watermark-signal-diminishing case.
However, in non-asymptotic situations, our simulation studies in Section \ref{sec: simulation} and language model experiments in Section \ref{sec: experiment} reveal that it is more robust in practice, especially when the text length is small.

Another application of maximal coupling is speculative decoding \citep{leviathan2023fast, chen2023accelerating}, where a smaller draft model generates speculative token predictions that are then accepted or corrected by a more powerful target model. We interpret this setup as a model of human-LLM interaction, where the smaller model represents the watermarked language model's text generation capabilities, and the target model represents the human's attempt to refine the generated text, focusing solely on enhancing the overall text quality. In Section \ref{sec: experiment}, we apply this setup to our watermarking scheme and demonstrate that our method effectively retains the watermark signal despite the targeted attempts to modify the text. Compared to \citet{aaronson}'s approach, our method results in fewer text modifications, suggesting that users would likely need to invest less editing effort to achieve their desired text quality when working with our watermarking scheme.\\
\textbf{Our Contributions.} We introduce a novel unbiased variant of the green/red list watermark by leveraging maximal coupling. After formulating the robust watermark detection problem as a sparse detection problem, we propose several detection methods and explore its theoretical properties. In particular, we provide theoretical insights into the statistical power of these detection methods, corroborated by extensive simulation studies. We demonstrate the effectiveness of our watermarking scheme against multiple baselines via experiments on large language models of varying sizes, as well as under a targeted text modification setup. Our findings offer a promising direction for watermarking in language models, balancing the need for detectability with minimal impact on text quality.

\section{Preliminaries}
\label{sec: prelim}

\textbf{Unbiased Watermarking Schemes.}
\cite{kirchenbauer2023watermark} introduced a ``soft'' watermark to mitigate the bias of the ``hard'' watermark by sampling from both the green list $\cG$ and the red list $\cR := \cW \setminus \cG$, but skewing the sampling distribution to favor tokens from the green list. That is, given a constant $\delta > 0$ and the next token distribution $P_w$, we construct the alternative distribution $Q_{w}$ as $
    Q_{w} = (\mathrm{e}^{\delta}P_w)/C$ if $w \in \cG$ and $Q_{w} = P_w/C$ if $w \not \in \cG$, 
where $C = \mathrm{e}^{\delta}P_{\cG} + P_{\cW\setminus \cG} = 1 + (\mathrm{e}^{\delta} - 1) P_{\cG}$. 
Both watermarks are biased according to the following definition \citep{kuditipudi2023robust,li2024statistical}.\footnote{In this setup, $\cS$ as a function of the next token distribution $\bm P$ and a random variable $\xi = \cG$, the (pseudo)randomly generated green list. For the proof of the biasedness, we refer the reader to the online Appendix.} 
\begin{definition}
    \label{definition:unbiasedness}
    A decoder $\cS$ is \textbf{unbiased} with respect to $\xi$ if for any NTP distribution $\bm P$ and any token $w \in \cW$, we have $\bP_{\cS, \xi}(\cS(\bm P, \xi) = w) = P_w$.\footnote{\citet{li2024statistical} requires the decoder $\cS$ is fully deterministic. We extend it slightly by allowing it to be stochastic and requiring its randomness to be independent of $\xi$.}
\end{definition}
\citet{aaronson} adapts exponential minimal sampling and achieves unbiasedness when the random variables $\xi_t$ are truly i.i.d. However, in practice, the pseudorandom variable $\xi_t$ is deterministically dependent on the prior $k$ tokens $w_{(t-k)}, ..., w_{(t-1)}$. This can lead to repetitive texts and quality degradation, particularly when the same $k$ tokens appears multiple times in the generated text \citep{kuditipudi2023robust}. 
While \citet{hu2023unbiased, dathathri2024scalable} propose repeated context masking to prevent applying watermark on step $t$ if the prior $k$ tokens have been used to watermark previously, our experiments in Section \ref{sec: experiment} demonstrate that this decoder still produces text that is more repetitive and lower quality compared to other watermarked decoders. Several theoretical analyses and variants of the deterministic-decoder-based method have been proposed \citep{fernandez2023three, zhao2024permute}, but no existing work completely addresses this issue. In contrast, our proposed stochastic decoder incorporates additional randomness for sampling and the pseudorandom variable $\zeta_t$ is only used for rejection. Another method is the inverse transform sampling strategy \citep{christ2023undetectable, kuditipudi2023robust,hu2023unbiased,li2024statistical}. However, these methods typically yield weaker power than the exponential minimal sampling. 
Building on inverse transform sampling, \citet{wu2023dipmark} proposed an unbiased watermarking scheme that incorporates random permutation and a probability reweighting strategy.\footnote{Given a pseudorandom permutation of the vocabulary $\cW$ and a parameter $\alpha < 0.5$, this scheme samples from the distribution $\bm Q$, where $Q_{w^{(i)}} = F_{w^{(i)}} - F_{w^{(i-1)}}$ and
$F_{w^{(i)}} = \max\left\{\sum_{j = 1}^i P_{w^{(j)}} - \alpha, 0\right\} + \max\left\{\sum_{j = 1}^i P_{w^{(j)}} - (1 - \alpha), 0\right\}.$
Here, $w^{(i)}$ denotes the $i$th token in the permutation, with the convention that $F(w^{(0)}) = 0$.} While this method is unbiased, tuning the hyperparameter can be a delicate task – a limitation that our proposed method eliminates. Furthermore, as shown in Section \ref{sec: experiment}, our method achieves higher detection power. For a more extensive review of unbiased watermarking schemes, see the online Appendix.

\begin{algorithm}
    \caption{Maximal Coupling for Token Sampling}
    \label{alg: token_sampling_coupling}
    \begin{algorithmic}
    \STATE \textbf{Input:} Token distributions $\bm P, \bm Q$, random variable $\zeta \in [0,1]$
    \STATE \textbf{Output:} Sampled token $w$
    \STATE Compute $p = \sum_{w \in \cW} \min(P_{w}, Q_{w})$.
    \IF{$\zeta \le p$}
        \STATE Sample $w$ from the normalized overlap distribution $\propto \min(\bm P, \bm Q)$
    \ELSE
        \STATE Sample $w$ from the normalized excess distribution $\propto \max(0, \bm P - \bm Q)$
    \ENDIF
    \STATE \textbf{return} $w$
    \end{algorithmic}
\end{algorithm}
\textbf{Maximal Coupling.} Maximal coupling is a powerful tool in probability theory that constructs a coupling between two random variables to maximize the probability that they are equal \citep{thorisson2000coupling}. We use maximal coupling to sample tokens in a way that preserves the original next-token distribution $\bm P_t$ while incorporating information from an alternative distribution $\bm Q_t$, both supported on the vocabulary $\cW$. Algorithm~\ref{alg: token_sampling_coupling} outlines the procedure to sample $w$ from the NTP distribution $\bm P$, where both $\min$ and $\max$ operations are applied element-wise, and the proof for Lemma \ref{lem:maximal_coupling_unbias} can be found in the online Appendix.
\begin{lemma}
    \label{lem:maximal_coupling_unbias}
    For any distributions $\bm P, \bm Q$, and $w$ sampled in this way, if $\zeta$ is a standard uniform variable, i.e. $\zeta \sim U[0,1]$, we have $w \sim \bm P$.
\end{lemma}

\section{Watermarking Scheme via Maximal Coupling}
\label{sec: scheme}

We formulate the watermarking procedure within a communication protocol involving three parties: the language model provider, the user, and the detector. As outlined in Algorithm~\ref{alg: communication_protocol}, this protocol implies that a watermarking scheme has two parts: (1) a {\it decoding} (or sampling) algorithm that can embed identifiable statistical signals into the generated text, and (2) a {\it detection} algorithm that can detect the watermark signal in the received text. 
In this remainder of this section, we introduce every aspect of our proposed watermark, including different detection methods, asymptotic theoretical analysis on detection boundaries, and practical considerations.

\begin{algorithm}[!hbt]
\caption{Communication Protocol}
\label{alg: communication_protocol}
\begin{algorithmic}
    \STATE 1. The language model provider shares a random watermark key $\bm \xi$ with the detector.
    \STATE 2. The user sends a prompt $\cP$ to the language model provider.
    \STATE 3. The language model provider generates a text of length $l$, denoted by $w_{1:l} = (w_1, ..., w_l)$, given the prompt $\cP$ and the watermark key $\bm \xi$.
    \STATE 4. The user sends the text $\tilde{w}_{1:n}$, which may be either (i) (an edited version of) the generated text $w_{1:l} = (w_1, ..., w_l)$
    or (ii) a piece of text independent of $w_{1:l}$ (e.g., text that the user wrote without using the watermarked language model).
    \STATE 5. Using the watermark key $\bm \xi$, the detector determines if $\tilde{w}_{1:n}$ contains the embedded watermark.
    \end{algorithmic}
\end{algorithm}
\textbf{Setup.} For clarity in illustrating the watermarking scheme, we first describe a simplified setup for the communication protocol and delay the adaptation in practice to Section \ref{sec:watermark_key}. In particular, we assume that: (1) the text generation has a predefined length $n$; (2) the language model provider randomly generates the watermark key $\bm \xi = (\cG_{1:n}, \zeta_{1:n})$ and shares it with the detector: $\cG_{1:n} = (\cG_1, \ldots, \cG_n)$ is a set of green lists, and $\zeta_{1:n} = (\zeta_1,...,\zeta_n)$ consists of independent standard uniform random variables; (3) the language model provider uses $\cG_{1:n}$ and $\zeta_{1:n}$ to generate a single output text $w_{1:n} = (w_1, ..., w_n)$ given an arbitrary prompt $\cP$; and (4) the user only modifies the text via substitution attack, i.e., the user replaces some tokens in the generated text with other tokens so that $l = n$ in this protocol.

\subsection{Decoding}
\label{sec: decoding}

\begin{algorithm}
\caption{Decoding Algorithm}
\label{alg: decoding}
\begin{algorithmic}
\STATE \textbf{Input:} Watermark key $\bm \zeta = (\cG_{1:n}, \zeta_{1:n})$, language model $\cM$, prompt $\cP$
\STATE \textbf{Output:} Decoded token sequence $w_{1:n} =\{w_1, w_2, ..., w_n\}$
\FOR{$t = 1$ to $n$}
    \STATE Obtain token distribution $\bm P_t$ from the language model $\cM(\cP, w_1,...,w_{t-1})$
    \STATE Define $\bm Q_t$ as in Equation \eqref{eq:hard green list}, i.e. $Q_{t,w} = P_{t,w} \mathbb{1}_{w \in \cG_t}/P_{t,\cG_t}$.
    \STATE Generate $w$ by maximal coupling (Algorithm~\ref{alg: token_sampling_coupling}) with $\bm P = \bm P_t$, $\bm Q = \bm Q_t$, and $\zeta = \zeta_t$.
    \STATE Set $w_t = w$.
\ENDFOR
\end{algorithmic}
\end{algorithm}
The decoding algorithm works as follows: for $t = 1, ..., n$, given the token distribution $\bm P_t$ at step $t$, we first construct an alternative distribution $\bm Q_t$ as in Equation \eqref{eq:hard green list}. Then we use $\zeta_t$ to sample the next token $w$ via maximal coupling. This process is summarized in Algorithm~\ref{alg: decoding}. As an edge case, if $P_{t,\cG_t} = 0$, we simply sample from $\bm P_t$. Note that this decoding scheme is unbiased by Lemma \ref{lem:maximal_coupling_unbias}. As $\min(P_{t,w}, Q_{t,w}) = P_{t,w} \mathbb{1}_{w \in \cG_t}$ and $\max(0, P_{t,w} - Q_{t,w}) = P_{t,w} \mathbb{1}_{w \not \in \cG_t}$, our watermarked decoder is effectively a stochastic function $\cS(\bm P_t, \cG_t, \zeta_t)$ sampling from the following distribution:
$$
\bP[w_t = w | \bm P_t, \cG_t, \zeta_t] = \mathbb{1}_{\zeta_t \le P_{t,\cG_t}}\frac{P_{t,w} \mathbb{1}_{w \in \cG_t}}{P_{t,\cG_t}} + \mathbb{1}_{\zeta_t > P_{t,\cG_t}}\frac{P_{t,w} \mathbb{1}_{w \not \in \cG_t}}{1 - P_{t,\cG_t}},
$$
and we arrive at the following lemma validating the intuition behind the statistical signal.

\begin{lemma}
\label{lem:conditioned_on_green}
    Let $\bm P_t$ be the original token distribution at step $t$ and $w_t$ the next token sampled from the decoding scheme described in Algorithm~\ref{alg: decoding}, with $\zeta_t$ from a standard uniform distribution. Then the conditional distribution of $\zeta_t$ given the event $\{w_t \in \cG_t\}$ is uniform on $\left[0, P_{t,\cG_t}\right].$
\end{lemma}

\begin{remark}
    \label{rem: red list}
    By symmetry, we can also see that $\zeta_t$, conditioned on that the next token $w_t$ is red, follows a uniform distribution on $\left[P_{t,\cG_t}, 1\right],$ i.e., $(1 - \zeta_t | w_t \text{ is red}) \sim$ $U\left[0, 1 - P_{t,\cG_t}\right]$.
\end{remark}

\begin{remark}
    Our decoder can also be viewed as an extension of that proposed by \cite{christ2023undetectable}, which developed a watermarking technique for binary language models ($\mathcal{M}^{(b)}$) where the vocabulary is limited to two tokens ($\mathcal{W}^{(b)} = \{0, 1\}$). In this binary context, the watermark key $\zeta$ in Algorithm~\ref{alg: token_sampling_coupling} directly determines whether the next sampled token is $0$ or $1$, making our method functionally equivalent to their proposal. To adapt this binary model to accommodate larger vocabularies ($|\mathcal{W}| \gg 2$), \cite{christ2023undetectable} suggested representing each token as a $\log|\mathcal{W}|$-bit string and applying binary language models to decode these strings. However, this adaptation lacks practicality since language models operate natively with large vocabularies. In contrast, our approach avoids the extra step of manually decomposing token probabilities into bits by directly leveraging the model’s inherent probability distribution. Our method provides a more natural extension by working directly with these large vocabulary sets rather than requiring binary encoding and decoding steps, by conceptualizing all green tokens as one category and all red tokens as another.
\end{remark}

\subsection{Detection}
\label{sec: detection}
As we assume that the detector has perfect knowledge of the green lists, they can identify all the green tokens in the received text $\tilde{w}_{1:n}$. In this section, we restrict our attention to only detecting the watermark signals from the green tokens; as pointed out in Remark \ref{rem: red list}, the detection method can be easily extended to include the red tokens.\\
\textbf{Detection as Hypothesis Testing.} Let $m$ be the number of green tokens in $\tilde{w}_{1:n}$. For simplicity, we assume that the first $m$ tokens in $\tilde{w}_{1:n}$, namely $\tilde w_1, ..., \tilde w_m$, represent the green tokens. If no corruption was made to the generated text, we can reformulate the detection objective into the following hypothesis testing problem:

\begin{enumerate}
    \item[$H_0$:] the text $\tilde{w}_{1:m}$ is independent of the decoder; i.e.
    $\zeta_t|(\bm P_t, \cG_t) \overset{\iid}{\sim} U[0,1]$ for $t = 1, ..., m$.
    \item[$H_1$:] the text $\tilde{w}_{1:m}$ is generated from the decoder; $\zeta_t|(\bm P_t, \cG_t) ~\sim~ U\left[0, P_{t,\cG_t}\right]$ for $t = 1, ..., m$.
\end{enumerate}
\textbf{Sparse Mixture Detection.} In practice, the user may modify the generated text. @hen a token is substituted, the corresponding uniform random variable $\zeta$ would be independent of the variable whether the new token is from the green list. In this case, the statistical signal of the watermark becomes sparse, leading us to consider an alternative hypothesis where a fraction $\varepsilon_m$ of the $\zeta_t$'s still present the watermark signal. Thus, we propose the following:

\begin{enumerate}
    \item[$H_1^{(\mathrm{mix})}$:] the text $\tilde{w}_{1:m}$ is first generated from the decoder, and then modified by the user via substitution; that is, $\zeta_t|(\bm P_t, \cG_t) ~\sim~ (1 - \varepsilon_m)U[0,1] + \varepsilon_mU\left[0, P_{t,\cG_t}\right]$ for $t = 1, ..., m$.
\end{enumerate}

First, we investigate the minimum level of watermark signal presence, that is $\varepsilon_m$, required to differentiate between watermarked and non-watermarked texts. In fact, if $\varepsilon_m$ is vanishingly small, then any kind of detection is impossible:

\begin{theorem}
    \label{thm:impossible_detection}
    Let $0 < r, p < 1$, $\varepsilon_m = m^{-p}$ and $P_{t,\cG_t} \ge m^{-r}$ for all $t$. If $2p - r > 1$, then $H_0$ and $H_1^{(\mathrm{mix})}$ merge asymptotically. That is, for any test, the sum of type I and type II errors tends to 1 as $m \to \infty$.
\end{theorem}

To prove such results, it suffices to consider the case where each $P_{t,\cG_t} = m^{-r}$, as the smaller $P_{t,\cG_t}$ is, the more watermark signal is present in the $\zeta_t$'s. For our analysis, we reduce the alternative hypothesis into $\iid$ mixtures of uniform random variables and uniform random variables on $[0, m^{-r}]$. This type of $\iid$ assumption for alternative hypotheses is extensively analyzed in the context of sparse detection problems \citep{donoho2004higher, donoho2015higher, cai2014optimal}. Theorem \ref{thm:impossible_detection} is a direct adaptation of the results in these works. 

On the other hand, if $\varepsilon_m$ is not too small, then with some mild assumptions, it is possible to separate the hypotheses successfully as the following theorem shows.

\begin{theorem}
    \label{thm:powerful_detection}
    Let $0 < p, q < 1$ and $\varepsilon_m = m^{-p}$.
    \begin{enumerate}
        \item Let $q + 2p < 1$ and $P_{t,\cG_t} \le 1 - m^{-q}$ for all $t = 1, ..., m$. Then $H_0$ and $H_1^{(\mathrm{mix})}$ separate asymptotically. Furthermore, the alternative hypothesis can be reliably detected using the likelihood ratio test.
        \item Let $q + 2p > 1$, and $P_{t,\cG_t} \ge 1 - m^{-q}$ for all $t = 1, ..., m$. Then $H_0$ and $H_1^{(\mathrm{mix})}$ merge asymptotically. No test can reliably separate $H_0$ and $H_1^{(\mathrm{mix})}$.
    \end{enumerate}
\end{theorem}
The assumptions on $P_{t,\cG_t}$ indicate that when the sum of the probabilities for green tokens is bounded away from $1$, watermark signals become detectable. Conversely, if $P_{t,\cG_t}$ is nearly $1$, then $U[0, P_{t,\cG_t}]$ approximates $U[0,1]$, making the signal from $\zeta_t \sim U[0, P_{t,\cG_t}]$ nearly indistinguishable from that of a modified token. We will discuss the implications of these findings for language models at the end of this subsection; for now, we focus on the hypothesis testing problem. \\
\textbf{Sum Test.} Since the probability sum of the green tokens, $P_{t,\cG_t}$'s, are unknown to the detector, the likelihood ratio test cannot be used in practice. We begin by considering a simple test based on the sum $s = \sum_{i=1}^m \zeta_m$. Under the null hypothesis, $\mathbb{E}_0 s = m/2$, while under the alternative hypothesis, $\mathbb{E}_1 s$ is smaller than $m/2$ by a factor depending on $m$ and parameters controlling the signal strength, e.g. $\varepsilon_m$ and $P_{t,\cG_t}$. With this observation, we can reject $H_0$ if $s$ is smaller than some critical value $c'$ satisfying $\mathbb{E}_1 s \le c' \le\mathbb{E}_0 s $. We call this test the sum test, and it can reliably detect the alternative hypothesis under slightly stronger assumptions than those stated in Theorem \ref{thm:powerful_detection}, as stated in the next proposition.
\begin{proposition}
    \label{prop:powerful_detection_sum}
    Let $0 < p, q < 1$ and $\varepsilon_m = m^{-p}$.
    \begin{enumerate}
        \item Let $q + p < 1/2$ and $P_{t,\cG_t} \le 1 - m^{-q}$ for all $t = 1, ..., m$. For the sum test with critical value $c' := \mathbb{E}_0s - m^{1 - (p + q)}/4$, the sum of type I and type II errors tends to 0 as $m \to \infty$.
        \item Let $q + p > 1/2$ and $P_{t,\cG_t} \ge 1 - m^{-q}$ for all $t = 1, ..., m$. For the sum test with any critical value satisfying $\mathbb{E}_1s \le c' \le \mathbb{E}_0s$, the power of the test is bounded away from 1.
    \end{enumerate}
\end{proposition}

Here $c' = \mathbb{E}_0s - m^{1 - (p + q)}/4$ is a convenient choice of the critical value for theoretical analysis. In practice, since $s$ is approximately normally distributed under the null hypothesis, we can use the quantile of the normal distribution to determine the critical value. \\
\textbf{Higher Criticism.} In the literature on sparse mixture detection, higher criticism (HC) \citep{donoho2004higher, donoho2015higher} is recognized as a powerful non-parametric method: in cases where the likelihood ratio test asymptotically separates the hypotheses, HC achieves the same separation, a property known as ``adaptive optimality.'' This insight is also relevant in our context. HC proceeds as follows: for $t = 1, \dots, m$, we begin by ordering the $\zeta_t$ values in ascending order:
$
\zeta_{(1)} < \zeta_{(2)} < ... < \zeta_{(m)}.
$
Then we define the test statistic as:
$$
\mathrm{HC}^*_m = \max_{1 \le t \le m} \mathrm{HC}_{m,t}, \quad \mathrm{HC}_{m,t} = \sqrt{m}\cdot \left(\frac{t/m - \zeta_{(t)}}{\zeta_{(t)}(1-\zeta_{(t)})}\right),
$$
and reject the null hypothesis when $\mathrm{HC}^*_m$ is large. HC is very different from conventional moment-based test statistics, such as the sum of the $\zeta_t$'s discussed in Proposition \ref{prop:powerful_detection_sum}. The key observation is that under the null hypothesis, $\zeta_t \overset{\iid}{\sim} U[0,1]$ so that $\mathrm{HC}_{m,t} \approx N(0,1)$ and thus $\mathrm{HC}^*_m \approx \sqrt{2 \log\log m}$, which grows to infinity extremely slowly. However, under the alternative hypothesis, when some of the $\zeta_t$'s are drawn from $U[0, P_{t,\cG_t}]$ for $P_{t,\cG_t} < 1$, $\mathrm{HC}^*_m$ will tend to be large. The following theorem shows that HC matches the power of the likelihood ratio test in the context of our problem.

\begin{theorem}[Adaptive optimality]
    \label{thm:powerful_detection_HC}
    Let $0 < p, q < 1$ and $\varepsilon_m = m^{-p}$. If $q + 2p < 1$ and $P_{t,\cG_t} < 1 - m^{-q}$ for all $t = 1, ..., m$, then for any constant $\delta > 0$, the decision rule that rejects the null hypothesis if and only if
    $
    \mathrm{HC}^*_m \ge \sqrt{(2+\delta) \log\log m},
    $
    has its sum of type I and type II errors tend to 0 as $m \to \infty$.
\end{theorem}

Theorem \ref{thm:powerful_detection_HC} demonstrates that HC achieves adaptive optimality for detecting our watermark in the considered asymptotically challenging setting. This asymptotic perspective aligns with established literature on sparse detection problems \citep{donoho2004higher, donoho2015higher}. While the finite-sample convergence of HC is often investigated through simulations \citep{li2015higher, arias2017distribution}, a gap remains between practical finite-sample performance and asymptotic theoretical guarantees. Our simulation and language model experiment results suggest that this adaptive optimality may converge slowly, consistent with previous findings on HC \citep{gontscharuk2015intermediates}.

Since the statistic $\mathrm{HC}^*_m$ has a heavy-tailed distribution under the null hypotheses, and stringent tests based on heavy-tailed null distributions can lead to lower power in small samples, \cite{donoho2004higher,donoho2015higher} suggests a modified HC statistic involving the maximum over all $\zeta_t$'s greater than or equal to $1/m$, which is 
$
\mathrm{HC}^+_m = \max_{1 \le t \le m, \zeta_{(t)} \ge 1/m} \mathrm{HC}_{m,t} $
to enhance test power. We incorporate this modification in our experiments in subsequent sections. Additionally, we aim to select a critical value that controls type I errors at a specific significance level, such as $\alpha = 0.01$. Following \citet{tony2011optimal}, a practical approach for choosing the critical value involves simulating the $\mathrm{HC}^+_m$ scores under the null hypothesis multiple times and using the top $\alpha$-percentile from the resulting empirical histogram. Given the slow weak convergence of the HC statistic, this method typically yields more accurate critical values than the theoretical consideration $\sqrt{(2+\delta) \log\log m}$.\\
\textbf{Practical Considerations.} Lastly, we emphasize that assuming $P_{t,\cG_t}$ is bounded away from 1 is not necessarily practical for language models: the probability of sampling a token from the green list can be arbitrarily close to 1, or even equal to 1, especially when some parts of the texts are deterministic. Nonetheless, previous results can be reinterpreted in a way that reflects the practical situation for language models: notice that if for a green token, the probability sum of its corresponding green list $P_{t,\cG_t}$ is arbitrarily close to 1, then the corresponding $\zeta_t$ is sampled from a uniform distribution that is very close to $U[0,1]$. We can treat this $\zeta_t$ as if its corresponding token has been modified, i.e., this $\zeta_t$ does not belong to the set of $\varepsilon_m$ fraction of the $\zeta_t$'s that contribute to the watermark signal. In other words,  Proposition \ref{prop:powerful_detection_sum} and Theorem \ref{thm:powerful_detection_HC} essentially says that as long as {\it a sufficient fraction of the green tokens are sampled from the language model when $P_{t,\cG_t}$ is bounded away from 1}, the watermark signal can reliably be detected by the test based on the sum test or HC.

It is not necessary to restrict our attention solely to the green tokens. As noted in Remark \ref{rem: red list}, without user modifications, the random variable $\zeta_t$ corresponding to a red token follows a uniform distribution on $\left[P_{t,\cG_t}, 1\right]$, i.e., $(1 - \zeta_t|w_t \text{ is red})$ is a uniform random variable in the interval $[0, 1 - P_{t,\cG_t}]$. Hence, we can define the test statistics as 
\begin{align*}
    \zeta'_t := \zeta_t \mathbb{1}_{\{w_t \text{ is green}\}} + (1 - \zeta_t) \mathbb{1}_{\{w_t \text{ is red}\}},~
    s:= \sum_{t = 1}^n \zeta'_t,~\text{and}~\mathrm{HC}^*_n:= \max_{1 \le t \le n} \sqrt{n}\cdot \left(\frac{t/n - \zeta'_{(t)}}{\zeta'_{(t)}(1-\zeta'_{(t)})}\right),
\end{align*}
and the analogous guarantees from Proposition \ref{prop:powerful_detection_sum} and Theorem \ref{thm:powerful_detection_HC} will follow.

\subsection{Watermark Key}
\label{sec:watermark_key}

In cases where the watermarked text has a predetermined maximum length $n$, the language model provider can pre-generate both the green lists $\cG_{1:n}$ and uniform random variables $\zeta_{1:n}$, sharing them with the detector as the watermark key. However, users may request multiple texts of varying lengths. To accommodate this, we adopt the approach of \citet{kirchenbauer2023watermark}, where the language model provider supplies two distinct pseudorandom functions\footnote{For this work, a pseudorandom function can be understood as a deterministic function of the seed that appears ``random'' in downstream applications. For an exact definition, see \citet{vadhan2012pseudorandomness}.} as the watermark key: the first generates the green lists $\cG_t$, and the second generates the uniform random variables $\zeta_t$, for each $t = 1, \dots, n$. These functions use the previous $k$ tokens $w_{(t-k)}, \dots, w_{(t-1)}$, or the prior $k$-grams, as the inputs.

During detection, we select tokens carefully to compute the test statistics. While repeated $k$-grams generate the same seed, scoring only unique $k$-grams would be too restrictive, especially for stochastic decoders, where the same $k$-gram context may yield different next tokens. Instead, we score tokens based on the uniqueness of the full $(k+1)$-tuple, comprising the watermark context ($k$-gram) and current token. This approach prevents double-counting of identical subsequences while capturing valid watermark signals from stochastic decoding. As \citet{fernandez2023three} demonstrated, this scoring scheme improves true positive rates compared to only scoring unique $k$-grams, while maintaining accurate false positive rate estimates. This improvement results from capturing more valid watermark signals without redundant scoring. While further refinement of watermark key usage and sharing is possible, these considerations are beyond the scope of this paper \citep{christ2023undetectable, kirchenbauer2023reliability}. Finally, we note that a single pre-generated green list $\cG$ could be reused for all $t = 1, \dots, n$, instead of generating a new green list at each $t$ \citep{zhao2023provable}. In theory, the unbiasedness of our decoding algorithm (Lemma \ref{lem:maximal_coupling_unbias}) is independent of the specific green list, a property also observed in practice; details are in the online Appendix.

\subsection{Interpret Speculative Decoding as Post-processing}

In the computer science literature, speculative decoding is a promising tool to accelerate LLM inference without sacrificing quality~\citep{leviathan2023fast, chen2023accelerating}. The key idea is to leverage a faster but smaller draft model to generate speculative token predictions ${\bm Q}$, which are then validated by the slower yet more powerful target model in parallel. Via the maximal coupling defined in Algorithm~\ref{alg: token_sampling_coupling}, with ${\bm P}$ being the next token distribution of the larger model, this method maintains the exact token distribution of the larger model. 

Speculative decoding can be viewed as a model of the interaction between a human user and a language model: the smaller draft model represents the language model's text generation capabilities, while the larger target model represents the human's attempt to refine and improve the generated text. Just as a human user may find most of the text generated by a language model satisfactory but still make targeted modifications to align the text more closely with their intended meaning, style, and coherence, the larger model in speculative decoding validates and modifies the speculations made by the smaller model. 

While previous studies employed i.i.d. random substitution attacks to modify watermarked tokens \citep{fernandez2023three, kuditipudi2023robust}, the post-processing in the speculative decoding setup offers a more realistic simulation of the interaction between a human user and a language model. By embedding a watermark only on the smaller draft model, we can explore its robustness under targeted modifications aimed at improving text quality. This practical setup also highlights an important consideration: while users may modify generated text, a high-quality watermarked model should minimize the need for such refinements, as frequent modifications risk compromising the watermark and lead to user fatigue. This analysis complements our theoretical analysis of the watermarking scheme and provides insights into the practical implications of the watermarking technique. The details on the implementation of our speculative decoding setup can be found in the online Appendix.
\section{Simulation Studies}
\label{sec: simulation}

We conduct simulation studies in two distinct scenarios to explore the robust detection boundaries discussed in Section \ref{sec: detection}. The first scenario focuses on situations where each green token carrying the watermark signal exhibits a sufficient signal strength, i.e., we only require $P_{t,\cG_t} \ge m^{-r}$. This setting allows us to examine when detection is impossible, as revealed by Theorem \ref{thm:impossible_detection}. In contrast, the second scenario deals with a weak signal condition, i.e., $P_{t,\cG_t} = 1 - m^{-q}$. Here, we aim to assess the effectiveness of both the sum test and HC in detecting signals, as detailed in Proposition \ref{prop:powerful_detection_sum} and Theorem \ref{thm:powerful_detection_HC}. For all simulations, we increase the number of green tokens $m$ from $10^2$ to $10^5$, by a factor of 10, and set the significance level $\alpha = 0.01$. To ensure a fair comparison, we compute each test statistic under both the null and alternative hypotheses 2000 times and use the quantile of the simulated histogram of the test statistic under the null hypothesis to determine the critical value. We then compute the rejection rate, or the power, by calculating the proportion of test statistics under the alternative hypothesis that exceeded the critical value. Furthermore, we use HC$^+_m$ as the test statistic for HC, as suggested in Section \ref{sec: detection}.\\
\textbf{The First Regime: Strong Signal.}
We fix the value of $r$ at 0.2 and test different values of $p \in \{0.25, 0.5, 0.75\}$. To generate random variables $\zeta_t$ under the alternative hypothesis, we first sample $P_{t,\cG_t}$ uniformly from $[m^{-r},1]$, and then sample $\zeta_t$ from $U[0, P_{t,\cG_t}]$. For context, $100^{-0.2} = 0.398, 1000^{-0.2} = 0.251, 10000^{-0.2} = 0.158, 100000^{-0.2} = 0.1$. The results are shown in Figure \ref{fig:impossible_detection}. When $p$ is small, a large fraction of the $\zeta_t$'s carry the watermark signal, and the power of both tests converges to 1 as $m$ increases. However, as $p$ becomes larger, the power of both tests decreases.
Indeed, with $p = 0.75$, $r=0.2$ makes the problem undetectable as $2p - r = 1.2 > 1$, and the power of both tests vanishes, consistent with Theorem \ref{thm:impossible_detection}. Interestingly, the power of the sum test is higher than that of HC, especially when $m$ is smaller. This observation will recur in the second regime.
\begin{figure}[t]
    \caption{The rejection rate of the sum test and HC under the first regime. As $p$ increases, the power of both tests vanishes.}
    \centering
    \includegraphics[width=0.325\textwidth]{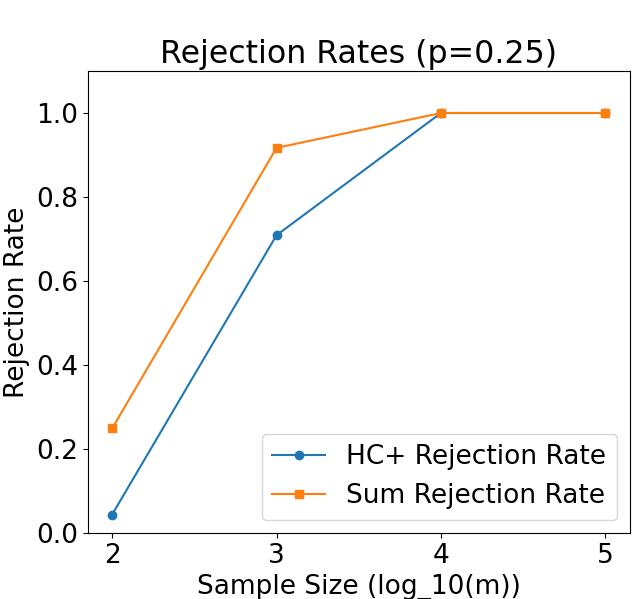}
    \includegraphics[width=0.325\textwidth]{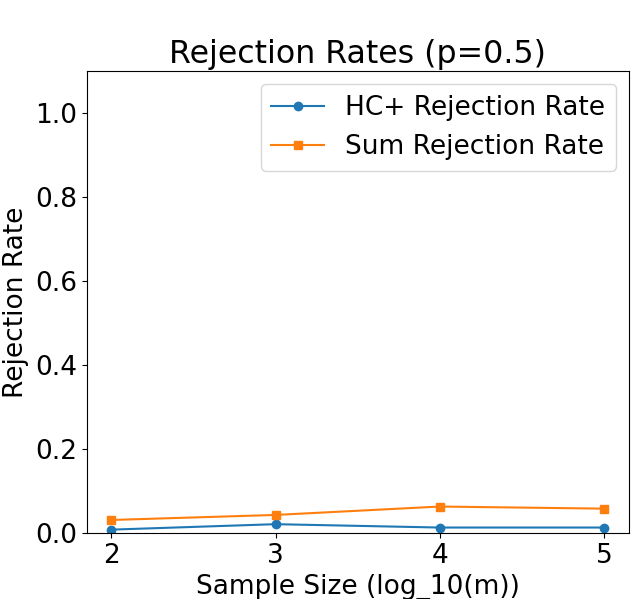}
    \includegraphics[width=0.325\textwidth]{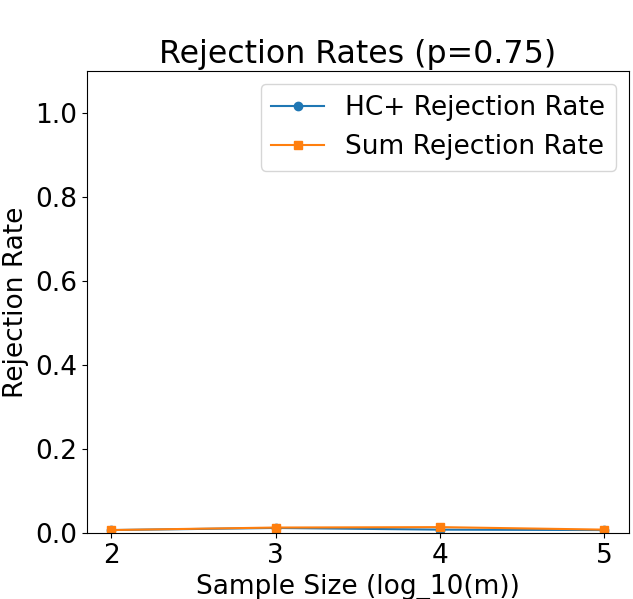}
    \label{fig:impossible_detection}
\end{figure}
\newline\textbf{The Second Regime: Weak Signal.} We will verify the two theoretical boundaries, illustrated by the $p + q = 1/2$ and $2p + q = 1$ as shown in Figure \ref{fig:detection_boundary}. 
The shaded area below the $p + q = 1/2$ boundary represents the cases where the sum test asymptotically separates the two hypotheses. The shaded area under the line $2p + q = 1$ denotes the cases where HC asymptotically separates the two hypotheses. Each dot in the figure denotes a problem instance with the $(p, q)$ parameters accordingly selected. These selections enable a comprehensive evaluation of both tests across four distinct scenarios. For all simulations, we set $P_{t,\cG_t} = 1 - m^{-q}$ to maintain a consistently weak signal strength.
\begin{figure}[t]
\caption{The detection boundaries $2p + q = 1$ and $p + q = 1/2$ as shown in Theorem \ref{thm:powerful_detection_HC} and Proposition \ref{prop:powerful_detection_sum}. The blue shaded area under the line $p + q = 1/2$ represents the detectable regime for the sum tests. The orange shaded area under the line $2p + q = 1$ shows the regime which is detectable for HC but not for the sum test. The red dots represent the choices of the parameters $p$ and $q$ we investigate in this simulation experiment.}
    \label{fig:detection_boundary}
    \centering
    \includegraphics[width=0.7\textwidth]{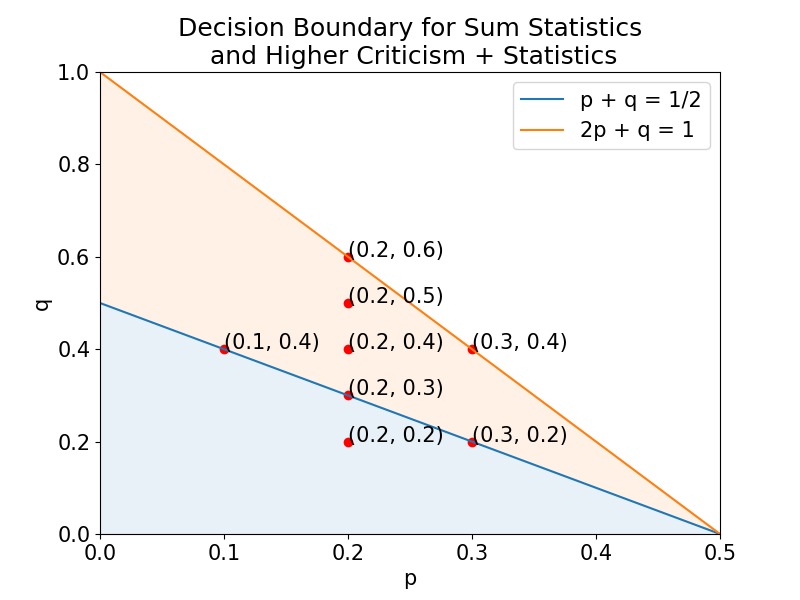}
\end{figure}
\begin{figure}[hbt!]
\caption{The rejection rate of the sum test and HC under the second regime, scenario 1. When both $p$ and $q$ are small, the power of both tests tends to converge to 1. As we move towards the boundary line $p + q = 1/2$, the power of the sum test hovers around 0.3, while the power of HC increases to 1. Continuing into the orange area, the sum test shows diminishing power, whereas HC's power continues to rise, albeit at a slower pace. Upon crossing the upper boundary $2p + q = 1$, both tests' power trends towards zero.}
    \centering
    \includegraphics[width=0.325\textwidth]{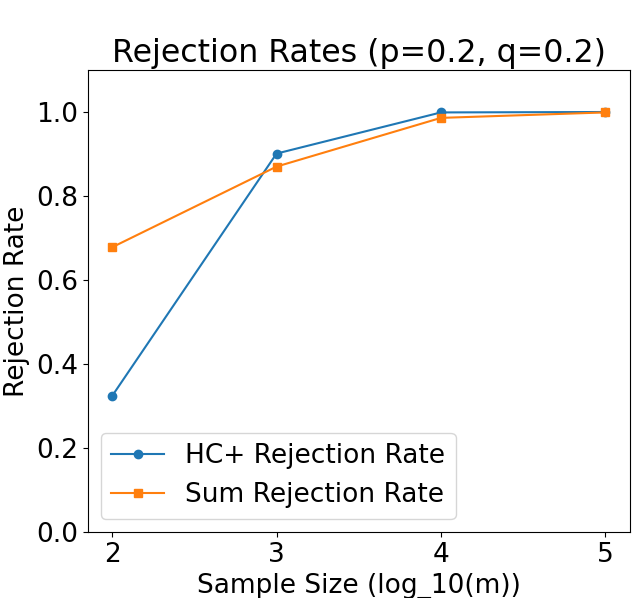}
    \includegraphics[width=0.325\textwidth]{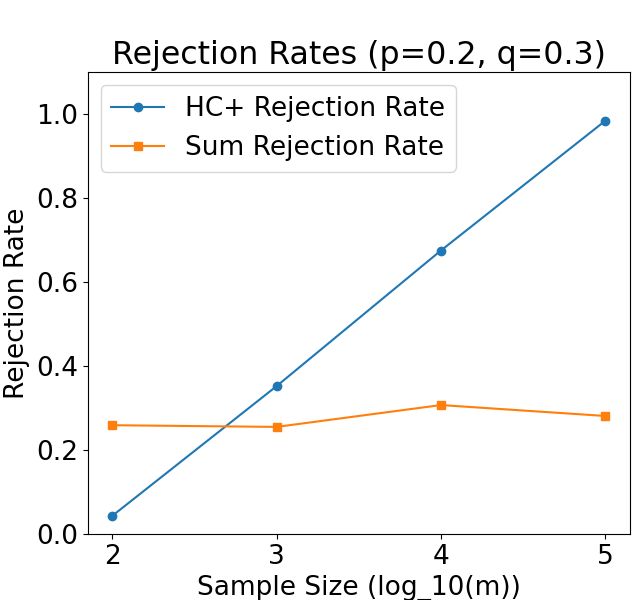}
    \includegraphics[width=0.325\textwidth]{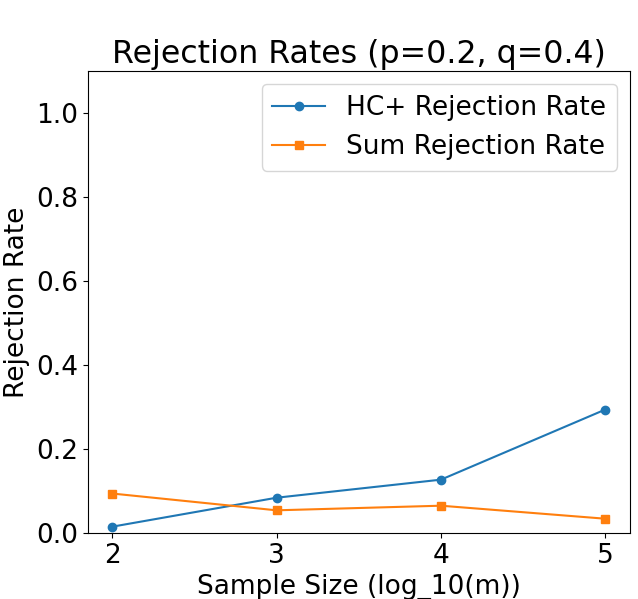}
    \includegraphics[width=0.325\textwidth]{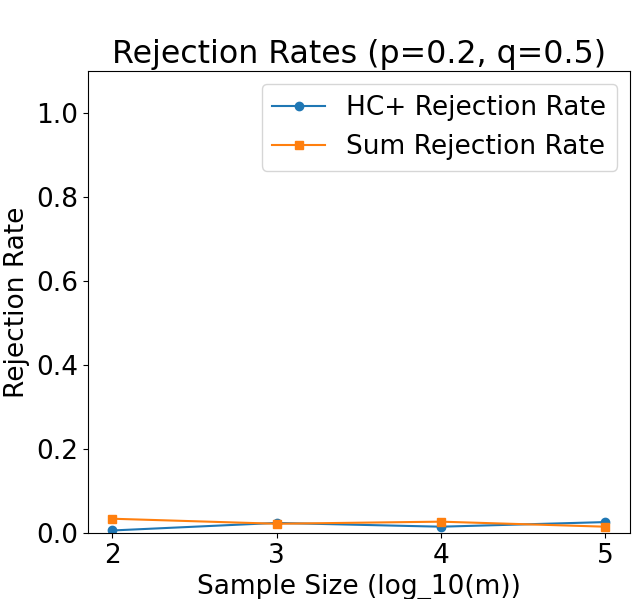}
    \includegraphics[width=0.325\textwidth]{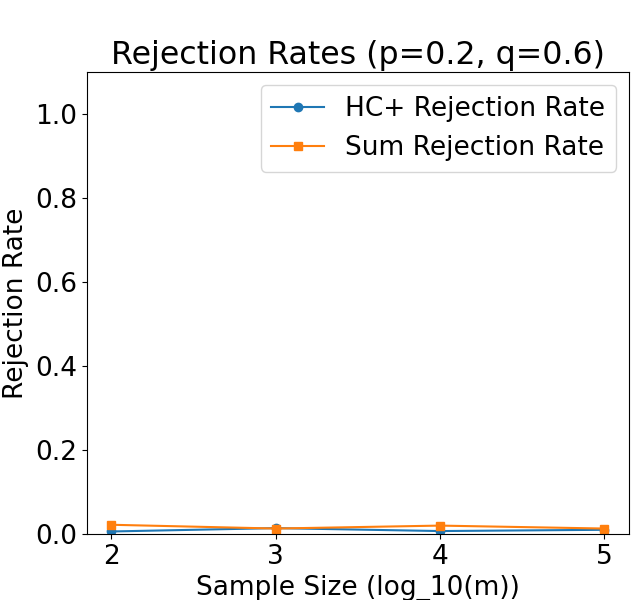}
    \label{fig:scenario_1}
\end{figure}
\newline\textit{Scenario 1: Fix $p$ and vary $q$.} We set the parameter $p$ at 0.2 and vary the parameter $q \in \{0.2, 0.3, 0.4, 0.5, 0.6\}$. Essentially, we examine how well each test performs as we maintain the proportion of signals but reduce their strength. The findings are illustrated in Figure \ref{fig:scenario_1}. When both $p$ and $q$ are small, the power of both tests tends to converge to 1. As we move towards the boundary line $p + q = 1/2$, the power of the sum test hovers around 0.3, while the power of HC increases to 1. Past this boundary, the sum test shows diminishing power, whereas HC's power continues to rise, albeit at a slower pace. Upon reaching the upper boundary $2p + q = 1$, both tests' power trends towards zero. Interestingly, in scenarios with smaller $m$, the sum test often outperforms HC, consistent with our observation in the first regime, hinting the sum test may be more suitable for analysis involving shorter texts.
\begin{figure}[t]
    \caption{The rejection rate of the sum test and HC under the second regime, scenario 2. At the boundary line $p + q = 1/2$, the power of the sum test hovers around 0.3, while the power of HC increases to 1. As we move towards the boundary $2p + q = 1$, both tests' power trends towards zero.}
    \label{fig:scenario_2}
    \centering
    \includegraphics[width=0.325\textwidth]{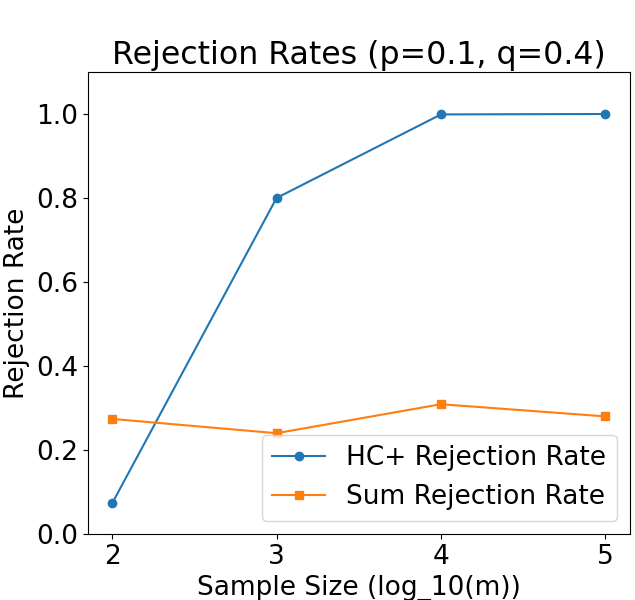}
    \includegraphics[width=0.325\textwidth]{figs/rejection_rates_p_0.2_q_0.4.png}
    \includegraphics[width=0.325\textwidth]{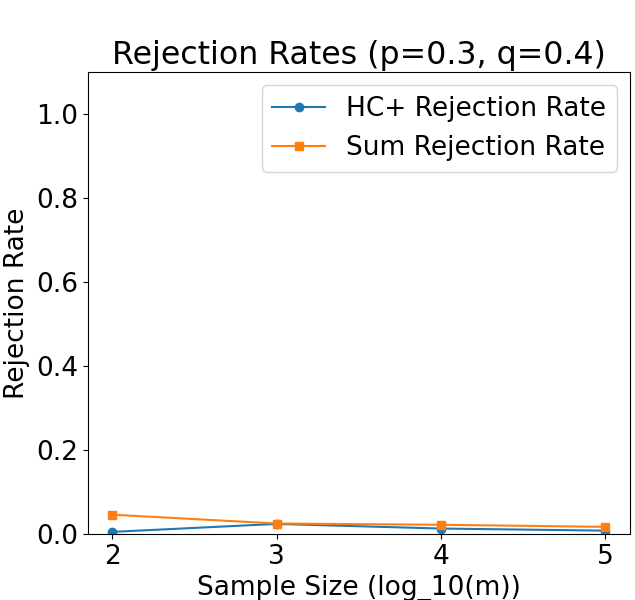}
\end{figure}
\newline\textit{Scenario 2: Fix $q$ and vary $p$.} We set the parameter $q$ at 0.4 and adjust the parameter $p \in \{0.1, 0.2, 0.3\}$. This approach allows us to investigate the effectiveness of both statistical tests as we maintain the strength of the signal but reduce the proportion of signals. The results are shown in Figure \ref{fig:scenario_2}. The observation is similar to that in Scenario 1. At the boundary line $p + q = 1/2$, the power of the sum test hovers around 0.3, while the power of HC increases to 1. As we move towards the boundary $2p + q = 1$, both tests' power trends towards zero, consistent with Theorem \ref{thm:impossible_detection}.
\begin{figure}[t]
    \caption{The rejection rate of the sum test and HC under the second regime, scenario 3. Across all situations, the power of the sum test stays around 0.3, while the power of HC approaches 1 more slowly as $p$ increases.}
    \centering
    \includegraphics[width=0.325\textwidth]{figs/rejection_rates_p_0.1_q_0.4.png}
    \includegraphics[width=0.325\textwidth]{figs/rejection_rates_p_0.2_q_0.3.png}
    \includegraphics[width=0.325\textwidth]{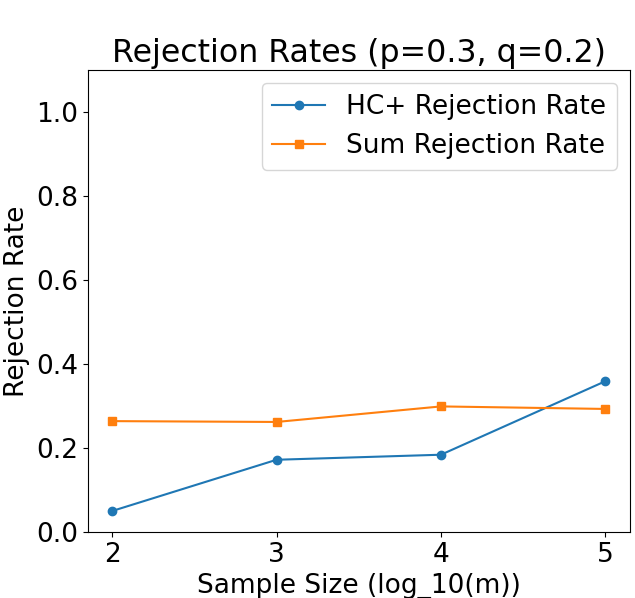}
    \label{fig:scenario_3}
\end{figure}
\newline\textit{Scenario 3: On the boundary $p + q = 1/2$.} We keep $p+q = 1/2$ and vary $p$ in $\{0.1, 0.2, 0.3\}$. That is, we stay on the detection boundary of the sum test while adjusting both the proportion and the strength of the signal. The results are shown in Figure \ref{fig:scenario_3}. Across all situations, the power of the sum test stays around 0.3, while the power of HC approaches 1 more slowly as $p$ increases. With increasing $p$, the proportion of signals decreases dramatically: e.g. $1000^{-0.1} = 0.501, 1000^{-0.2} = 0.251, 1000^{-0.3} = 0.126$. This indicates that the effectiveness of HC is more sensitive to changes in the proportion of watermarked tokens (as a function of $p$) than to the strength of the signal (as a function of $q$).
\newline\textit{Scenario 4: On the boundary $2p + q = 1$.} We keep $2p + q=1$ and vary $p$ in $\{0.2, 0.3\}$. The results are shown last plots in Figure \ref{fig:scenario_1} and \ref{fig:scenario_2}. On this boundary line, the power of both tests vanishes, consistent with Theorem \ref{thm:impossible_detection}.
\begin{figure}[t]
\caption{The histograms of the test statistics under the null and alternative hypotheses for the sum test and HC when $m = 100$. The test statistics are computed from simulation under both the null and alternative hypotheses 2000 times. The critical value is set by the quantile of the simulated histogram of the test statistic under the null hypothesis. }
    \label{fig:small_m}
    \centering
    \includegraphics[width=0.99\textwidth]{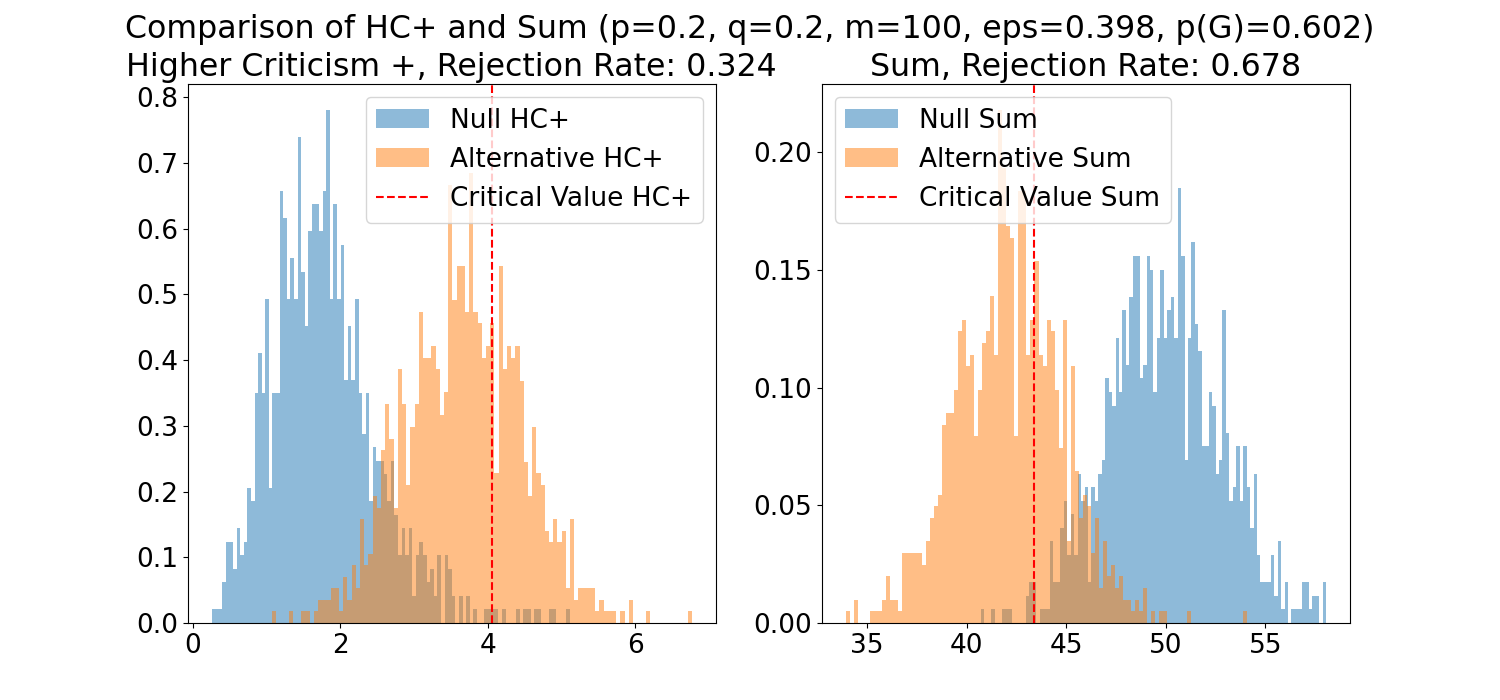}
\end{figure}

Interestingly, there is a consistent pattern across all scenarios: the sum test tends to outperform HC when $m$ is relatively small ($\log_{10} m \le 3)$. One explanation is that for higher criticism, even though HC$^+_m$ less ``heavy-tailed'' than HC$^*_m$, the distribution of HC$^+_m$ under the null hypothesis is still right-skewed, as shown in Figure \ref{fig:small_m}. When setting a small significance level, the critical value—derived from the quantile of the simulated histogram of the test statistic under the null hypothesis—can be relatively high and results in a low power. In contrast, under the null hypothesis, the sum of $\zeta_t$'s is approximately symmetrically distributed. This distinction contributes to the sum test's higher power. The same pattern will recur in the language model experiments in the next section.

\section{Language Model Experiments}
\label{sec: experiment}

\textbf{Setup.} We evaluate our watermarking method against three schemes discussed above: (1) the Gumbel-max watermark based on exponential minimal sampling by \citet{aaronson}, (2)  the soft green/red list watermark by \citet{kirchenbauer2023watermark} with a moderate text distortion hyperparameter $\delta = 1$, and (3) the Distribution-Preserving watermark (DiPmark) by \citet{wu2023dipmark} with their recommended watermark hyperparameter $\alpha = 0.45$. For all methods where applicable, we set $\gamma = 0.5$ to control the green list size such that $|\cG| = \gamma|\cW|$. For a fair comparison, our scheme uses both the green tokens and the red tokens for detection as described in Section \ref{sec: detection}. Following \cite{fernandez2023three}, for all methods, we score tokens based on the uniqueness of the full $(k + 1)$-tuple and compute $p$-values using the exact null distribution rather than $z$-scores, and flag a text as watermarked if the $p$-value is lower than $0.01$, which ensures a false positive rate to be less than $0.01$. To ensure fair comparison of text quality, we implement repeated context masking for all methods \citep{wu2023dipmark}.

As users typically prompt language models for open-ended conversational interactions, question-answering, and task assistance, the most widely-used language models are instruction fine-tuned. To mimic this, we prompt Microsoft's instruction fine-tuned \texttt{Phi-3} models \citep{abdin2024phi} on two question-answering datasets, namely ELI5 \citep{fan2019eli5} and FinQA \citep{maia201818}.\footnote{For each dataset, we use the same 200 samples selected by \citet{tu2023waterbench}.} For all experiments, we embed watermarks to the \texttt{Phi-3-mini-4k-instruct} model (3.8B parameter).

To evaluate the text distortion induced by the watermark, we use the S-BERT similarity score to compare the semantic similarity between the generated text of the same model with and without watermarking \citep{reimers2019sentence}; the higher the similarity score, the less the distortion. We also report the true positive rate (TPR) (i.e., the fraction of correctly identifying a watermarked response among all watermarked samples) before and after attacks, including a purely substitution-based attack (TPR aug.) and a paraphrasing attack (TPR para.).\footnote{Following \citet{fernandez2023three}, we simulate the substitution attack by randomly and independently replacing each token with probability $0.1$. For the paraphrasing attack, we apply the Context-aware Lexical Substitution \citep{yang2022tracing}, which introduces a BERT-based infill model for generating contextually appropriate lexical substitutions \citep{devlin2019bert} with paraphrasing ratio$=0.3$.} In addition, we investigate the speculative decoding setup to assess the robustness of different watermarks where the goal is to modify watermarked texts so as to improve their quality.
As discussed earlier, this setup better approximates the human-language model interaction. Here, we choose \texttt{Phi-3-medium-4k-instruct} (14B parameters) as the target model.\footnote{The source code is available at \url{https://github.com/Xieyangxinyu/Debiasing-Watermarks-for-Large-Language-Models-via-Maximal-Coupling}.}
We conduct comprehensive ablation studies with identical experimental configurations, also including Meta's \texttt{LLaMA} model family \citep{dubey2024llama}. Since both model families produce similar results, we focus on the \texttt{Phi-3} family in the following sections for brevity and present the complete results in the online Appendix.
\begin{table*}[t]
\captionsetup{font={stretch=1}}
\renewcommand{\arraystretch}{0.5}
\caption{Comparison of the S-BERT similarity scores and the true positive rates (TPR) among four watermarking schemes: (1) the Gumbel-max watermark \citep{aaronson}, (2) the green/red list watermark \citep{kirchenbauer2023watermark}, (3) the DiPmark method \cite{wu2023dipmark}, and (4) our proposed one. }
\label{tab: S-BERT and TPR, alpha 0.01, main}
\begin{small}
\centering
\begin{tabular}{c c c l l | r | r | r | r}
\toprule
Model & Data & $k$ & Metric & & Gumbel-max & Green/red list & DiPmark & Ours \\
\midrule
\multirow{20}{*}{\rotatebox[origin=c]{90}{Phi-3-mini-4k-instruct (3.8B)}} & \multirow{10}{*}{\rotatebox[origin=c]{90}{FinQA}} & \multirow{5}{*}{2} & \multirow{2}{*}{S-BERT} & mean & 0.7682 & 0.8167 & 0.7975 & 0.7963 \\
 &  &  &  & median & 0.8115 & 0.8472 & 0.8396 & 0.8382\\
 &  &  & TPR &  & 0.9500 & 0.6550 & 0.9250 & \textbf{0.9750}\\
 &  &  & TPR aug. &  & 0.9450 & 0.3500 & 0.8150 & \textbf{0.9750}\\
 &  &  & TPR para. &  & 0.9300 & 0.2800 & 0.7000 & \textbf{0.8400}\\
\cmidrule{3-9}
 &  & \multirow{5}{*}{4} & \multirow{2}{*}{S-BERT} & mean & 0.7791 & 0.8233 & 0.7967 & 0.8023 \\
 &  &  &  & median & 0.8312 & 0.8494 & 0.8405 & 0.8288\\
 &  &  & TPR &  & 0.9550 & 0.7750 & 0.9350 & \textbf{0.9500}\\
 &  &  & TPR aug. &  & 0.9550 & 0.5150 & 0.7600 & \textbf{0.8300}\\
 &  &  & TPR para. &  & 0.9250 & 0.2600 & 0.5500 & \textbf{0.6750}\\
\cmidrule{2-9}
 & \multirow{10}{*}{\rotatebox[origin=c]{90}{ELI5}} & \multirow{5}{*}{2} & \multirow{2}{*}{S-BERT} & mean & 0.7133 & 0.7178 & 0.7160 & 0.7161 \\
 &  &  &  & median & 0.7184 & 0.7239 & 0.7350 & 0.7290\\
 &  &  & TPR &  & 1.0000 & 0.9100 & \textbf{0.9900} & \textbf{0.9900}\\
 &  &  & TPR aug. &  & 1.0000 & 0.6500 & 0.9200 & \textbf{0.9350}\\
 &  &  & TPR para. &  & 1.0000 & 0.4900 & 0.8850 & \textbf{0.9300}\\
\cmidrule{3-9}
 &  & \multirow{5}{*}{4} & \multirow{2}{*}{S-BERT} & mean & 0.7182 & 0.7084 & 0.7123 & 0.7185 \\
 &  &  &  & median & 0.7286 & 0.7236 & 0.7266 & 0.7344\\
 &  &  & TPR &  & 1.0000 & 0.8850 & 0.9900 & \textbf{1.0000}\\
 &  &  & TPR aug. &  & 1.0000 & 0.5750 & 0.8750 & \textbf{0.9550}\\
 &  &  & TPR para. &  & 0.9950 & 0.4350 & 0.7500 & \textbf{0.9250}\\
\bottomrule
\end{tabular}
\end{small}
\end{table*}
\newline\textbf{Comparison against Baselines.} Table~\ref{tab: S-BERT and TPR, alpha 0.01, main} reports the S-BERT similarity scores and the TPR among the four watermarking schemes. Although the Gumbel-max watermark is theoretically unbiased with the assumption of perfect pseudorandomness \citep{aaronson}, experimental results show that the deterministic decoder can distort the generated text more severely than the stochastic ones. Indeed, when using the preceding $k$ tokens as the seed for the pseudorandom variables $\xi$ during decoding, perfect pseudorandomness is compromised in practice, as evidenced by the consistently lower S-BERT scores. Meanwhile, when comparing against the green/red list soft watermark and DiPmark, our method often leads to a more competitive TPR when similar S-BERT scores are achieved. This trend holds also under substitution-based and context-aware paraphrasing attacks, as reflected in the TPR aug. and TPR para. metrics.\\
\begin{table}[t]
\captionsetup{font={stretch=1}}
\renewcommand{\arraystretch}{0.5}
\caption{Comparative analysis of repeated token proportions across different watermarking methods.}
\label{tab:token_stats}
\centering
\begin{small}
\begin{tabular}{l l l l | r}
\toprule
Model & Dataset & $k$ & Method & Repeated (\%) \\
\midrule
\multirow{16}{*}{\rotatebox[origin=c]{90}{Phi-3-mini-4k-instruct (3.8B)}} & \multirow{8}{*}{FinQA} & \multirow{4}{*}{2} & Gumbel-max & 18.61\% \\
 &  &  & Green/red list & 16.77\% \\
 &  &  & DiPmark & 17.80\% \\
 &  &  & Ours & 17.03\% \\
\cmidrule{3-5}
 &  & \multirow{4}{*}{4} & Gumbel-max & 4.17\% \\
 &  &  & Green/red list & 3.71\% \\
 &  &  & DiPmark & 3.50\% \\
 &  &  & Ours & 4.04\% \\
\cmidrule{2-5}
 & \multirow{8}{*}{ELI5} & \multirow{4}{*}{2} & Gumbel-max & 15.34\% \\
 &  &  & Green/red list & 13.55\% \\
 &  &  & DiPmark & 13.32\% \\
 &  &  & Ours & 14.78\% \\
\cmidrule{3-5}
 &  & \multirow{4}{*}{4} & Gumbel-max & 3.78\% \\
 &  &  & Green/red list & 1.81\% \\
 &  &  & DiPmark & 2.14\% \\
 &  &  & Ours & 2.07\% \\
\bottomrule
\end{tabular}
\end{small}
\end{table}
\textbf{Text Repetition.} We quantitatively evaluate the repetitiveness issue discussed in Section \ref{sec: experiment} by analyzing the proportion of repeated $k$-token sequences in the generated text. The ``Repeated'' column in Table \ref{tab:token_stats} reports the average percentage of times that repeated context masking is applied within each generated text, which directly reflects how frequently $k$-grams are repeated during generation. While repeated context masking was introduced by \citet{hu2023unbiased, dathathri2024scalable} to mitigate quality degradation, our results show that the Gumbel-max approach still produces notably more repetitive text compared to other methods. For example, with $k=2$ on the FinQA dataset, Gumbel-max exhibits an 18.61\% token repetition rate versus 16.77-17.80\% for other methods. This suggests that in practice, deterministic dependence on prior $k$-tokens in methods like Gumbel-max can bias the generation toward repetitive patterns that persist even with repeated context masking. In contrast, our stochastic decoder's additional randomness in sampling helps reduce such repetitive patterns while maintaining watermark effectiveness.\\
\textbf{Speculative Decoding.} Table \ref{tab:speculative_stats} presents a comparison of the TPR between the Gumbel-max watermark \citep{aaronson} and our method under the speculative decoding setup. This setup, where a draft model proposes tokens that may be rejected by the target model, naturally mirrors how users iteratively refine generated text - making targeted, non-independent edits to improve quality rather than random token substitutions. The average rejection rates indicate the proportion of watermarked tokens that are discarded during this process. While both approaches achieve comparable TPR, the lower rejection rate of our method can translate to improved user experience - users need to make fewer revisions to achieve their desired output compared to the Gumbel-max approach.\\
\textbf{Comparison of Test Statistics.}
In Table \ref{tab:stats}, we experiment with the two different choices of test statistics: the sum test and HC. Our experiments confirm the insights from Section \ref{sec: simulation}: the sum test consistently achieves higher TPR than HC across all configurations. This aligns with our simulation findings that the sum test outperforms HC when number of tokens used to detect the watermark is relatively small ($\le 10^3$). For longer generated texts, we expect HC to become increasingly competitive with the sum test, as suggested by our simulation results.
\begin{table*}[t]
\captionsetup{font={stretch=1}}
\renewcommand{\arraystretch}{0.5}
\caption{Comparison of the true positive rates (TPR) and rejection rates between the Gumbel-max watermark and our proposed scheme under the speculative decoding setup.}
\label{tab:speculative_stats}
\centering
\begin{small}
\centering
\begin{tabular}{l l l l | r | r}
\toprule
Model & Data & $k$ & Metric & Gumbel-max & Ours \\
\midrule
\multirow{8}{*}{Phi-3-mini-4k-instruct (3.8B)} & \multirow{4}{*}{FinQA} & \multirow{2}{*}{2} & Avg. Rejection Rate & 20.97\% & 11.48\% \\
 &  &  & TPR & 0.3100 & 0.2850 \\
\cmidrule{3-6}
 &  & \multirow{2}{*}{4} & Avg. Rejection Rate & 22.35\% & 12.23\% \\
 &  &  & TPR & 0.4550 & 0.4600 \\
\cmidrule{2-6}
 & \multirow{4}{*}{ELI5} & \multirow{2}{*}{2} & Avg. Rejection Rate & 28.36\% & 15.44\% \\
 &  &  & TPR & 0.4150 & 0.4650 \\
\cmidrule{3-6}
 &  & \multirow{2}{*}{4} & Avg. Rejection Rate & 29.81\% & 16.47\% \\
 &  &  & TPR & 0.4450 & 0.4800 \\
\bottomrule
\end{tabular}
\end{small}
\end{table*}

\begin{table*}[t]
\captionsetup{font={stretch=1}}
\renewcommand{\arraystretch}{0.5}
\caption{Comparison of different test statistics for our watermarking scheme. Num. Tokens denotes the average number of tokens used to detect the watermark. The test statistics are the sum test and HC. The sum test is more powerful than HC.}
\label{tab:stats}
\centering
\begin{small}
\centering
\begin{tabular}{c c c l | r | r}
\toprule
Model & Data & $k$/Num. Tokens & Metric & Sum & HC$^+$ \\
\midrule
\multirow{8}{*}{Phi-3-mini-4k-instruct (3.8B)} & \multirow{4}{*}{FinQA} & \multirow{2}{*}{2/355.210} & TPR & \textbf{0.975} & 0.92 \\
 &  &  & TPR aug. & \textbf{0.87} & 0.71 \\
\cmidrule{3-6}
 &  & \multirow{2}{*}{4/370.115} & TPR & \textbf{0.95} & 0.91 \\
 &  &  & TPR aug. & \textbf{0.83} & 0.625 \\
\cmidrule{2-6}
 & \multirow{4}{*}{ELI5} & \multirow{2}{*}{2/299.045} & TPR & \textbf{0.99} & 0.98 \\
 &  &  & TPR aug. & \textbf{0.935} & 0.865 \\
\cmidrule{3-6}
 &  & \multirow{2}{*}{4/329.300} & TPR & \textbf{1} & 0.99 \\
 &  &  & TPR aug. & \textbf{0.955} & 0.83 \\
\bottomrule
\end{tabular}
\end{small}
\end{table*}
\section{Conclusions}
\label{sec: conclusion}

By integrating maximal coupling into the green/red list watermarking scheme, we achieve a provably unbiased method that is less susceptible to distortion under pseudorandomness. Moreover, by reformulating the watermark detection problem as a mixture detection problem and focusing on a sparse detection case, our theoretical analyses establish the asymptotic detection boundaries of our method. Our empirical results underscore the efficacy of our approach, showing superior preservation of generation distribution and detection power compared to previous methods. The findings not only enhance our understanding of effective watermarking in language models but also pave the way for future research in developing more robust, model-agnostic watermarking schemes.

\section*{Acknowledgment}
This work was supported in part by the National Key R\&D Program of China (2022YFA1007900), NSF DMS-2310679, a Meta Faculty Research Award, Wharton AI for Business, and the National Natural Science Foundation of China (12271013,72342004). This work utilized computational resources provided by the Argonne Leadership Computing Facility and was also based upon work supported by Laboratory Directed Research and Development (LDRD) funding from Argonne National Laboratory, under Contract No.\ DE-AC02-06CH11357. The authors would like to thank two anonymous referees for their constructive comments that helped improve the presentation of the paper.

\section*{Disclosure Statement} The authors report there are no competing interests to declare.

\bibliographystyle{plainnat}
\bibliography{bibfile}

\appendix
\section{Appendix}

\subsection{Unbiased Watermarking Schemes and Related Works}
\label{sec: related works}

\citet{kirchenbauer2023watermark} introduced a ``soft'' watermark to mitigate the bias of the ``hard'' watermark by sampling from both the green list $\cG$ and the red list $\cR := \cW \setminus \cG$, but skewing the sampling distribution to favor tokens from the green list. That is, given a constant $\delta > 0$ and the next token distribution $P_w$, we construct the alternative distribution $Q_{w}$ as

\begin{equation}
    \label{eq:prompt green list}
    Q_{w} = \begin{cases}
        \frac{\mathrm{e}^{\delta}P_w}{C} \quad &\text{if } w \in \cG\\
        \frac{P_w}{C} &\text{otherwise}
    \end{cases}
\end{equation}
where $C = \mathrm{e}^{\delta}P_{\cG} + P_{\cW\setminus \cG} = 1 + (\mathrm{e}^{\delta} - 1) P_{\cG}$ is the normalizing constant. Clearly, if the green list is pre-determined, then both the ``hard'' and the ``soft'' watermarks are skewing the original token distribution. We argue that even if the green list is generated uniformly at random for each $t$, the soft watermark is still biased. To see that the soft watermark can distort the text distribution, consider the following example: suppose we have a binary vocabulary (that is, we only generate tokens from \{0,1\}) and the original language model is a biased coin, which turns up head with probability $0.9$. Let the green list $\cG$ be sampled from $\{\{0\}, \{1\}\}$ with equal probability. After embedding the soft watermark with $\delta = 1$, the sampling probability of the next token becomes
\begin{align*}
    \bP_{\cG}[\text{next token is }1] &= \bP[\text{next token is }1| \cG= \{0\}]\bP[\cG= \{0\}] \\
    & \quad \quad \quad + \bP[\text{next token is }1| \cG= \{1\}]\bP[\cG= \{1\}] \\
    &= \frac{1}{2}\cdot \frac{0.9}{0.9 + \mathrm{e}\cdot 0.1} + \frac{1}{2}\cdot \frac{\mathrm{e}\cdot 0.9}{\mathrm{e}\cdot 0.9 + 0.1} \approx 0.8644
\end{align*}

This leads to a delicate trade-off between the text distortion and the detection power for this family of watermarking schemes, leaving the choice of $\delta$ a parameter that requires careful tuning in practice. For more discussion on such trade-off, we refer the reader to \citep{kirchenbauer2023watermark,kirchenbauer2023reliability, cai2024towards}. Our method, on the other hand, is unbiased by construction, thus avoiding this parameter-tuning issue.

Several alternative methods have been proposed to design provably unbiased watermarking schemes. The main idea involves modifying the decoding strategy of the language model decoder to create watermark signals without altering the token distributions. The first and widely recognized method adapts the exponential minimal sampling, as introduced by \citet{aaronson}. This method selects the next token $w_t$ directly by computing 
$w_t = \cS(\bm P_t, \zeta_t) := \argmax_{w \in \cW} \log U_{w}/P_w$
where $\zeta_t = (U_1,\ldots, U_{|\cW|})$ is the random variable that consists of $|\cW|$ i.i.d. copies of $U[0,1]$ and is part of the watermark key. This method is unbiased when true randomness is used for generating $\zeta_t$ and the proof involves the Gumbel trick \citep{gumbel1948statistical, fernandez2023three,li2024statistical}. However, a common practice to share the watermark key is to seed a pseudorandom number generator for $U$ with the prior $k$ tokens $w_{(t-k)}, ..., w_{(t-1)}$ (discussed in Section \ref{sec:watermark_key}). Notice that upon fixing the $U$, the decoder is deterministic, unlike the standard decoding methods we use in this paper. When the same $k$-gram appears multiple times in the generated text, this deterministic decoder will generate the same token following these $k$-grams, leading to repetitive texts and degraded text quality \citep{christ2023undetectable, kuditipudi2023robust}. To mitigate this, \citet{wu2023dipmark} propose to apple repeated context masking, which prevents the watermark from being applied on step $t$ if the same $k$-gram has been used to watermark previously. This method has later been adapted by \citet{hu2023unbiased, dathathri2024scalable}. However, the deterministic nature of this decoder can bias the generation toward repetitive patterns that persist even with repeated context masking, as we demonstrate in Section \ref{sec: experiment}. Several theoretical analyses and variants of this method have been proposed \citep{fernandez2023three, zhao2024permute}, but no existing work completely addresses this issue. 

Another method is the inverse transform sampling strategy \citep{christ2023undetectable, kuditipudi2023robust,hu2023unbiased,li2024statistical}. The idea is to first sample a standard uniform random variable $\zeta_{1:n} = (\zeta_1, ..., \zeta_n)$, and then map $\zeta_t$ to the next token. Although the randomness from $\zeta_t$ can guarantee the unbiasedness of the token generation, designing an effective test to detect watermarks with this method is challenging: \citet{kuditipudi2023robust} showed that their correlation-based detection method yields weaker power than the exponential minimal sampling method, even when perfect randomness is applied, while the detection method proposed by \citet{hu2023unbiased} involves solving an optimization problem and requires hyperparameter tuning to achieve competitive power.

Building on inverse transform sampling, \citet{wu2023dipmark} proposed the following watermarking scheme: before generating the next token, start with a random permutation of the token set $\cW$ and assign the last $0 < \gamma < 1$ fraction of the permutation to the green list $\cG$. Then, reweight the token distribution $P_w$ with a parameter $0 < \alpha < 1$ by the following strategy: first, for $i = 1,..., |\cW|$, define
$F_{w^{(i)}} = \max\{\sum_{j = 1}^i P_{w^{(j)}} - \alpha, 0\} + \max\{\sum_{j = 1}^i P_{w^{(j)}} - (1 - \alpha), 0\}$, where $w^{(i)}$ is the $i$th token in this permutation and $f(w^{(0)}) = 0$; then set 
$Q_{w^{(i)}} = F_{w^{(i)}} - F_{w^{(i-1)}}.$
Intuitively, this token distribution $\bm Q$ gives more weight to tokens towards the end of the permutation. This ensures unbiased token selection after integrating the randomness of the permutation and intuitively increases the likelihood of choosing tokens from the green list. During the detection, the detector can count the number of green tokens in the generated text and compare it with $\gamma$ times the total length of the text. Nonetheless, the reweighting parameter $\alpha$ is independent of the size $\gamma$ of the green list, meaning that reweighting doesn't necessarily favor all the green list tokens during generation, leading to a potentially noisy watermark signal. As a result, tuning these two hyperparameters, namely $\alpha$ and $\gamma$, becomes a delicate task. 

Notably, \citet{kuditipudi2023robust} advocates for the language model provider to generate the random variables $\zeta_{1:n} = (\zeta_1, ..., \zeta_n)$ before text generation and share them with the detector as part of the watermark key. The language model provider can then reuse this fixed $\zeta_{1:n}$ in clever ways to generate texts in variable lengths. When the generated text has been modified via substitution, insertion, or deletion, the detector can detect the watermark by aligning the received text with the shared random variables $\zeta_{1:n}$. We do not pursue this strategy in this paper, however, as it is not clear how the issue of multiple testing will arise from this reuse, and if so, how much it may weaken the power of the test. Moreover, aligning the received text with the shared random variables $\zeta$ may be computationally expensive for practical deployment, especially when the text is long.

Recently, several theoretical works have emerged on other aspects of watermarking for language models. \citet{huang2023towards} characterizes the Uniformly Most Powerful (UMP) watermark when a small amount of distortion is allowed, as well as the minimax Type II error in the model-agnostic setting. Nonetheless, the computational efficiency of these characterized tests remains unclear. \citet{zhang2023watermarks} demonstrates that under some assumptions, a computationally bounded attacker can erase the watermark without causing significant quality degradation. However, in many practical situations, like preventing AI-generated content from being used for further model training, existing watermarking schemes can still be useful.

\subsection{Total Variation Distance and Hellinger Distance}
Here we provide a brief review of the total variation distance and the Hellinger distance. For two probability measures $\bm P$ and $\bm Q$ on the same probability space $\mathcal{X}$, with densities $p$ and $q$ with respect to some underlying base measure $\nu$. The total variation distance is defined as
$$\|\bm P - \bm Q\|_{\text{TV}} = \frac{1}{2}\int_{\mathcal{X}}|p(x) - q(x)|d\nu(x).$$
The Hellinger distance is defined as
$$\|\bm P - \bm Q\|_{\text{H}} = \left(\int_{\mathcal{X}}\left(\sqrt{p(x)} - \sqrt{q(x)}\right)^2d\nu(x)\right)^{1/2}$$
Notice that from the definitions, 
$$\frac{1}{2}\|\bm P - \bm Q\|_{\text{H}}^2 \leq \|\bm P - \bm Q\|_{\text{TV}}.$$
and the Le Cam's inequality states that 
$$\|\bm P - \bm Q\|_{\mathrm{TV}} \leq H(\bm P \| \bm Q) \sqrt{1-\frac{H^2(\bm P \| \bm Q)}{4}}$$

For product measures with $\iid$ components, $\bm P^{1:n} = \otimes_{t=1}^n\bm P^t$ and $\bm Q^{1:n} = \otimes_{t=1}^n\bm Q^t$, we can decouple the Hellinger distance as follows:
$$\|\bm P^{1:n} - \bm Q^{1:n}\|_{\text{H}}^2 = 
2-2\left(1-\frac{1}{2} H^2\left(\bm P_1 \| \bm Q_1\right)\right)^n$$
Thus,
$$2 - 2 \exp\left(-\frac{n H^2\left(\bm P_1 \| \bm Q_1\right)}{2}\right)\le \| \bm P^{1:n} - \bm Q^{1:n}\|_{\text{H}}^2 \leq  n H^2\left(\bm P_1 \| \bm Q_1\right)$$

We often use the Hellinger affinity as the notation is more convenient at times, which is defined as
$$\text{Aff}(\bm P, \bm Q) = \int_{\mathcal{X}}\sqrt{p(x)q(x)}d\nu(x).$$
Notice that $\frac{1}{2}\|\bm P - \bm Q\|_{\text{H}}^2 = 1 - \text{Aff}(\bm P, \bm Q)$.
Putting everything together, we have the following lemma: 
\begin{lemma}
    \label{lem:hellinger}
    If $\text{Aff}(\bm P^{1:n}, \bm Q^{1:n}) = 1 + o(1/n)$, then $\|\bm P^{1:n} - \bm Q^{1:n}\|_{\text{TV}}^2 = o(1)$ as $n \to \infty$; and if $\text{Aff}(\bm P^{1:n}, \bm Q^{1:n}) = 1 - \omega(1/n)$, then $\|\bm P^{1:n} - \bm Q^{1:n}\|_{\text{TV}}^2 \to 1$ as $n \to \infty$.
\end{lemma}

\subsubsection{Proof for Lemma \ref{lem:maximal_coupling_unbias}}
\label{sec: proof for unbiasedness}
\begin{proof}
    Let $\mu(w) \propto \min(P_w, Q_w)$ for $w \in \cW$ and $\nu(w) \propto \max(0, P_w - Q_w)$ for $w \in \cW$. With $\mathcal{A} := \{w \in \cW: P_w \ge Q_w\}$ and $p = \sum_{w \in \cW} \min(P_w, Q_w)$, we have $\mu(w) = \min(P_w, Q_w)/p$ for $w \in \cW$, and $\nu(w) = (P_w - Q_w)/(1-p)$ for $w \in \mathcal{A}$ and $0$ otherwise. Then, for any $w \in \mathcal{A}, p \cdot \mu(w) + (1-p) \cdot \nu(w) = Q_w + P_w - Q_w = P_w$ and for any $w \not \in \mathcal{A}, p \cdot \mu(w) + (1-p) \cdot \nu(w) = P_w + 0 = P_w$. Therefore, $w \sim \bm P$.
\end{proof}

\subsection{Proof for the Detection Scheme}
\label{sec:proof_general_case}

First, we assume no modification was made to the generated text; we will discuss the substitution attack in subsequent sections. Recall that as we restrict our attention to the $m$ green tokens, we can reformulate the detection scheme into a hypothesis-testing problem:
\begin{enumerate}
    \item[$H_0$:] the text $\tilde{w}_{1:n}$ is independent of the decoder; i.e.
    $\zeta_t \overset{\iid}{\sim} U[0,1]$ for $t = 1, ..., m$.
    \item[$H_1$:] the text $\tilde{w}_{1:n}$ is generated from the decoder; $\zeta_t|(\bm P_t, \cG_t) ~\sim~ U\left[0, P_{t,\cG_t}\right]$ for $t = 1, ..., m$.
\end{enumerate}

\subsubsection{Order Statistics.} We first discuss the simplest case, where each $P_{t,\cG_t}$ is equal and bounded away from 1. Then we progressively generalize the proof. In this section, let $\zeta_{1:m} = (\zeta_1, ..., \zeta_m)$ and $\zeta_{(m)} = \max(\zeta_1, ..., \zeta_m).$ 

Formally, we assume that there exists some constant $\delta > 0$ such that $P_{t,\cG_t} = 1 - \delta$ for each $t = 1, ..., m$. Under this assumption, Neyman-Pearson Lemma implies that the likelihood ratio test is the uniformly most powerful (UMP). In particular, the likelihood ratio test is given by
$$\Lambda(\zeta_{1:n}) = \frac{\prod_{t = 1}^{m} \mathbb{1}_{\zeta_j \in \left[0, P_{t,\cG_t}\right]}/P_{t,\cG_t}}{\prod_{t = 1}^{m} \mathbb{1}_{\zeta_t \in \left[0, 1\right]}} = (1 - \delta)^{-m} \mathbb{1}_{\zeta_{(m)} \le 1 - \delta}$$
where $\mathbb{1}$ denotes the indicator function. This likelihood-ratio test provides the decision rule as follows: for some critical value $c$,
\begin{itemize}
    \item if $\Lambda(\zeta_{1:n}) \ge c$, then reject $H_0$;
    \item if $\Lambda(\zeta_{1:n}) < c$, then accept $H_0$.
\end{itemize}
This is equivalent to the following decision rule: for some critical value $c'$,
\begin{itemize}
    \item if $\zeta_{(m)} \le c'$, then reject $H_0$;
    \item if $\zeta_{(m)} > c'$, then accept $H_0$.
\end{itemize}
Thus, to construct the UMP test, we need to find a critical value $c'$ such that:
$$\bP_{H_0}(\zeta_{(m)} \le c') = \alpha$$
where $\alpha$ is the significance level. Since $\zeta_{(m)}$ is the maximum of $m$ i.i.d. uniform random variables, we have  
$$\bP_{H_0}(\zeta_{(m)} \le c') = \prod_{t=1}^m[\bP_{H_0}(\zeta_{t} \le c')]^m = (c')^m.$$
Hence, we can choose $c' = \alpha^{1/m}$ to obtain the desired significance level. The power of the test is given by
$$\bP_{H_1}(\zeta_{(m)} \le c') = \prod_{t=1}^m[\bP_{H_1}(\zeta_{t} \le c')]^m = \begin{cases}
    \left(\frac{c'}{1 - \delta}\right)^m = \frac{\alpha}{(1 - \delta)^m} \quad &\text{for } c' \le 1 - \delta\\
    1 &\text{for } 1 - \delta <  c' \le 1.
\end{cases}    
$$

For fixed $\alpha$ and $\delta$, as $m$ increases, $c' = \alpha^{1/m}$ will be greater than $1 - \delta$, which leads the power to be equal to 1. In the case where $P_{t,\cG_t} \le 1 - \delta$ for all $t$, the above argument gives a lower bound on the power of the same test. 

\subsubsection{Text Modifications.} Before we discuss more general settings for $P_{t,\cG_t}$, we first discuss the case where the user can modify the generated text only via substituting some tokens. In particular, we assume that the user can substitute any token in the generated text with any other token without the knowledge of the green lists or $\zeta_t$'s. When a token is substituted, the corresponding uniform random variable $\zeta$ becomes independent of whether the new token falls into the corresponding green list. Under this circumstance, the statistical signal of the watermark becomes sparse, leading us to consider an alternative hypothesis where a fraction $\varepsilon_m$ of the $\zeta_t$'s still present signals of the watermark. Formally, we have the following hypothesis:

\begin{enumerate}
    \item[$H_1^{(\mathrm{mix})}$:] the text $\tilde{w}_{1:n}$ is first generated from the decoder, and then modified by substitution attacks described above; that is, $\zeta_t|(\bm P_t, \cG_t) ~\sim~ (1 - \varepsilon_m)U[0,1] + \varepsilon_mU\left[0, P_{t,\cG_t}\right]$ for $t = 1, ..., m$.
\end{enumerate}

Throughout our discussion, we treat $m \varepsilon_m$ as an integer for simplicity. 

\subsubsection{Suboptimality of the Order Statistics.} We apply the same test as above, where we reject the null hypothesis if and only if $\zeta_{(m)} \le c'$. Mimicking the above analysis, we can reach the following lemma:

\begin{lemma}
    \label{lem:substitution_attack}
    In the presence of the substitution attack described above, the power of this test is given by
    $$\bP_{H_1}(\zeta_{(m)} \le c') = \begin{cases}
        \frac{\alpha}{(1 - \delta)^{m\cdot \varepsilon_m}} \quad &\text{for } c' \le 1 - \delta\\
        \alpha^{1 - \varepsilon_m} &\text{for } 1 - \delta <  c' \le 1.
    \end{cases}
    $$
\end{lemma}

\begin{proof}
    The proof is similar to the above analysis. When $c' \le 1 - \delta,$
    $$\bP_{H_1}(\zeta_{(m)} \le c') = \prod_{t=1}^m[\bP_{H_1}(\zeta_{t} \le c')]^m = (c')^{m(1 - \varepsilon_m)}\cdot \left(\frac{c'}{1 - \delta}\right)^{m \varepsilon_m} = \frac{\alpha}{(1 - \delta)^{m\varepsilon_m}}$$
    where the last equality follows from the fact that $c' = \alpha^{1/m}$.
    On the other hand, when $c' > 1 - \delta$, we can assume without loss of generality that only the first $m \varepsilon_m$ of the $\zeta_t$'s presents the watermark signal. As these $\zeta$'s are bounded above by $1 - \delta$, we have
    \begin{align*}
        \bP_{H_1}(\zeta_{(m)} \le c') &= \prod_{t=1}^{m \varepsilon_m}[\bP_{H_1}(\zeta_{t} \le c')]^m \cdot \prod_{t=m \varepsilon_m + 1}^m[\bP_{H_1}(\zeta_{t} \le c')]^m\\
        &= (c')^{m(1 - \varepsilon_m)} = \alpha^{1 - \varepsilon_m}.
    \end{align*}
\end{proof}

Notice that the power of the test is now upper bounded by $\alpha^{1 - \varepsilon_m}$. This means, unless $\varepsilon_m \to 1$, i.e., the fraction of modification becomes negligible as the length of the text increases, the power of the test will be suboptimal, leading to the next proposition.

\begin{proposition}
    \label{prop:order stat powerless}
    For a given significance level $\alpha < 1$, if $\varepsilon_m$ does not converge to $1$, then the power of the test based on the order statistic $\zeta_{(m)}$ is bounded away from 1.
\end{proposition}

\subsubsection{Undetectability.} We now lift the assumption that all $P_{t,\cG_t}$'s are equal in our subsequent discussion. 

\begin{proof}[Proof of Theorem \ref{thm:impossible_detection}]
    As observed earlier, we can assume without loss of generality that $P_{t,\cG_t} = P_{\cG} = m^{-r}$ for all $t$. By Lemma \ref{lem:hellinger}, it suffices to show that the Hellinger affinity between $U[0,1]$ and $(1 - \varepsilon_m)U[0,1] + \varepsilon_mU[0, P_{\cG}]$ behaves asymptotically as $1 + o(1/m)$. Let $f(y)$ and $g(y)$ be the densities of $U[0,1]$ and $U[0, P_{\cG}]$ respectively. As $f(y) = 1$ and $g(y) = P_{\cG}^{-1}\mathbb{1}_{x \in [0, P_{\cG}]}$ on $[0,1]$, the Hellinger affinity is given by
    \begin{align*}
        \mathbb{E}_0\sqrt{f(y)((1-\varepsilon_m)f(y) + \varepsilon_m g(y))} &= \mathbb{E}_0\sqrt{(1 - \varepsilon_m) + \varepsilon_mg(y)}\\
        &=\mathbb{E}_0\sqrt{1 + \varepsilon_m(g(y) - 1)}
    \end{align*}
    where $\mathbb{E}_0$ denotes the expectation under the null hypothesis. As $0 < r, p < 1$ and $2p - r > 1$, we must have $r < p$. Hence, $\varepsilon_m g(y) \le \varepsilon_mp(G)^{-1} = m^{-p}m^r \to 0$. This implies that for large enough $m$, $x:= \varepsilon_m(g(y) - 1)$ satisfies
    $$1 + \frac{x}{2} - \frac{x^2}{8} \le \sqrt{1 + x} \le 1 + \frac{x}{2}.$$
    Hence, as $\mathbb{E}_0g(y) = 1$, we have
    $$1 - \mathbb{E}_0\frac{(\varepsilon_m(g(y) - 1))^2}{8} \le \sqrt{1 + \varepsilon_m(g(y) - 1)} \le 1.$$
    It remains to show that $\mathbb{E}_0(\varepsilon_m(g(y) - 1))^2 = o(1/m)$: again, as $\mathbb{E}_0g(y) = 1$,
    \begin{align*}
        \mathbb{E}_0(\varepsilon_m(g(y) - 1))^2 &= \varepsilon_m^2\mathbb{E}_0((g(y) - 1))^2\\
        &= \varepsilon_m^2\left(\mathbb{E}_0(g(y)^2 - 1)\right)\\
        &= \varepsilon_m^2\left(\mathbb{E}_0p(G)^{-2}\mathbb{1}_{x \in [0, P_{\cG}]} - 1\right)\\
        &= \varepsilon_m^2\left(P_{\cG}^{-1} - 1\right) \\
        &= m^{-2p}(m^r - 1) = o(1/m)
    \end{align*}
    where the last equality follows from the assumption that $2p - r > 1$. 
\end{proof}

\subsubsection{Detection Boundary.}

\begin{proof}[Proof of Theorem \ref{thm:powerful_detection}]
    We first prove the first part of the theorem. In this case, we can restrict our attention to the case where each $P_{t,\cG_t} = P_{\cG} = 1 - m^{-q}$ without loss of generality, as smaller $P_{t,\cG_t}$ only present more watermark signals. By Lemma \ref{lem:hellinger}, it suffices to show that the Hellinger affinity between $U[0,1]$ and $(1 - \varepsilon_m)U[0,1] + \varepsilon_mU[0, P_{\cG}]$ behaves asymptotically as $1 - \omega(1/m)$. Let $f(y)$ and $g(y)$ be the densities of $U[0,1]$ and $U[0, P_{\cG}]$ respectively. As $\varepsilon_m (g(y) - 1) = m^{-p}\cdot m^{-q}/(1 - m^{-q}) \to 0$, similar to the last proof, it is enough to investigate the behavior of $\mathbb{E}_0(\varepsilon_m(g(y) - 1))^2$:
    \begin{equation}
        \label{eq: separate asymptotically}
        \begin{split}
            \mathbb{E}_0(\varepsilon_m(g(y) - 1))^2 &= \varepsilon_m^2\left(\mathbb{E}_0(g(y)^2 - 1)\right)\\
            &= \varepsilon_m^2\left(P_{\cG}^{-1} - 1\right) \\
            &= m^{-2p}\cdot \frac{m^{-q}}{1 - m^{-q}} = \omega(1/m)
        \end{split}
    \end{equation}
    because $q + 2p < 1$. Hence, $H_0$ and $H_1^{(\mathrm{mix})}$ separate asymptotically. To show that the alternative hypothesis can be reliably detected using the LRT, it is sufficient to show that the sum of type I and type II error probabilities tends to 0 as $m \to \infty$. Since the proofs are similar, we present the argument for the type I error under the null hypothesis. Let $\ell = \log(1 + \varepsilon_m(g(y) - 1))$ and $L_m = m \ell$. It suffices to show 
    $$\mathbb{E}_0 L_m \to - \infty \quad \text{and} \quad \frac{\Var_0(L_m)}{[\mathbb{E}_0 L_m]^2} \to 0$$
    Let $x := \varepsilon_m(g(y) - 1)$ and use the fact that $\log(1 + x) \le x - x^2/4$ for $x \in (-1, 1]$. For large enough $m$, we have
    $$\mathbb{E}_0 \ell \le \varepsilon_m\mathbb{E}_0(g(y) - 1) - \frac{\varepsilon_m^2}{4}\mathbb{E}_0(g(y) - 1)^2 = -\frac{\varepsilon_m^2}{4}\mathbb{E}_0(g(y) - 1)^2 = -\omega(1/m)$$
    where the last equality follows from Equation \eqref{eq: separate asymptotically}. This gives us $\mathbb{E}_0 L_m \to -\infty$. On the other hand, using the fact that $\log^2(1 + x) \le 2x^2$ for $x \in (-0.5, 1]$, we have for large enough $m$,
    $$\Var_0 \ell \le \mathbb{E}_0 \ell^2\le 2\varepsilon_m^2\mathbb{E}_0(g(y) - 1)^2 \le 8 |\mathbb{E}_0 \ell|$$
    This gives us $\Var_0 L_m/[\mathbb{E}_0 L_m]^2 \to 0$. 

    We now prove the second part of the theorem. In this case, we can restrict our attention to the case where each $P_{t,\cG_t} = P_{\cG} = 1 - m^{-q}$. Following the same steps, we have 
    $$\mathbb{E}_0(\varepsilon_m(g(y) - 1))^2 = \varepsilon_m^2(P_{\cG}^{-1} - 1) = m^{-2p}\frac{m^{-q}}{1 - m^{-q}} = o(1/m)$$
    as $q + 2p > 1$. Hence, $H_0$ and $H_1^{(\mathrm{mix})}$ merge asymptotically.
\end{proof}

\subsubsection{Sum Test.}
\begin{proof}[Proof of Proposition \ref{prop:powerful_detection_sum}]
    For both parts of the proposition, we can assume without loss of generality that each $P_{t,\cG_t} = P_{\cG} = 1 - m^{-q}$. 

    As for the first part, the proofs for showing that both type I and type II error probabilities tend to 0 are very similar, so we present the argument for the type II error under the alternative hypothesis as it involves slightly more technicality. Let $s = \sum_{i=1}^m\zeta_i$ and $c' = \mathbb{E}_0s - m^{1 - (p + q)}$. Then 
    $$\mathbb{E}_1 s = m\mathbb{E}_1\zeta_1 = m\left(\frac{1 - \varepsilon_m}{2} + \frac{\varepsilon_m\cdot P_{\cG}}{2}\right) = \frac{m}{2} -\frac{m^{1 -(p+q)}}{2}$$
    and 
    $$\Var_1 s = m\Var_1\zeta_1 = O(m)$$
    By Chebyshev's inequality, we have
    $$\bP_1(s \le c') \le \bP_1(|s - \mathbb{E}_1s| \le c' - \mathbb{E}_1s) = \bP_1(|s - \mathbb{E}_1s| \le m^{1 - (p + q)}/4) \le \frac{4\Var_1s}{m^{2 - 2(p + q)}} = o(1)$$
    where the last equality follows from the fact that $q + p < 1/2$ and $\Var_1s = O(m)$.

    Now, for the second part of the proposition, it suffices to observe that $\mathbb{P}_1(s \le \mathbb{E}_0s)$ is bounded away from 1. As $p + q > 1/2, \mathbb{E}_0s - \mathbb{E}_1s = m^{1 - (p + q)}/2 <  = m^{1/2}$ and $\Var_1s = O(m)$, which implies that as $m \to \infty, \mathbb{E}_0s$ is within one standard deviation of the center of the distribution of $s$ under the alternative hypothesis. Hence, the power of the test must be bounded away from 1.
\end{proof}

\subsubsection{Higher Criticism.}

\begin{proof}[Proof of Theorem \ref{thm:powerful_detection_HC}]
    Under the null hypothesis, $HC_m^*$ is equal to the extreme value of a normalized uniform empirical process in distribution, which satisfies
    $$\frac{HC_m^*}{\sqrt{2\log\log m}} \to 1, \quad \text{in probability},$$
    so $\mathbb{P}_0(HC_m^* \ge \sqrt{(2 + \delta)\log\log m}) \to 0$ as $m \to \infty$; see Theorem 1.1. in \cite{donoho2004higher}.
    Before we show that $\mathbb{P}_1(HC_m^* \le \sqrt{(2 + \delta)\log\log m}) \to 0$, we present an equivalent form of the higher criticism statistic. Consider the empirical cumulative distribution function
    $$F_m(x) = \frac{1}{m}\sum_{i=1}^{m}\mathbb{1}_{\zeta_i \le x},$$
    and its standardized version
    $$W_m(x) = \sqrt{m} \cdot \frac{F_m(x) - x}{\sqrt{x(1 - x)}}.$$
    Then, we observe that the higher criticism statistic satisfies
    $$HC_m^* = \max_{1 \le t \le m} \sqrt{m} \cdot\left(\frac{t/m - \zeta_{(t)}}{\zeta_{(t)}(1-\zeta_{(t)})}\right) = \max_{1 \le t \le m} \sqrt{m} \cdot\left(\frac{F_m(\zeta_{(t)}) - \zeta_{(t)}}{\zeta_{(t)}(1-\zeta_{(t)})}\right) = \sup_{x \in [0,1]}W_m(x).$$
    If $\zeta_t$'s are i.i.d. uniform random variables, then $\mathbb{E}F_m(x) = x$ for all $x \in [0,1]$; while if $\zeta_t$'s are independent uniform random variables in $[0, P_{t,\cG_t}]$, where $P_{t,\cG_t} \le 1 - m^{-q}$, then $\mathbb{E}F_m(1 - m^{-q}) = 1$. Hence, setting $w =1 - m^{-q}$,
    \begin{align*}
        \mathbb{E}_1W_m(x) &= \sqrt{m}\cdot\left(\frac{\mathbb{E}_1F_m(x) - x}{\sqrt{x(1 - x)}}\right)\\
        &= \sqrt{m}\cdot\left(\frac{\frac{1}{m}\sum_{i=1}^{m}\mathbb{E}_1\mathbb{1}_{\zeta_i \le x} - x}{\sqrt{x(1 - x)}}\right)\\
        &= \sqrt{m}\cdot\left(\frac{\varepsilon_m(1 - x)}{\sqrt{x(1 - x)}}\right)\\
        &= \sqrt{m}\cdot\left(\frac{m^{-p}m^{-q}}{\sqrt{(1 - m^{-q})m^{-q}}}\right)\\
        &= \sqrt{\left(\frac{m^{1-2p-q}}{1 - m^{-q}}\right)} = \Omega(n^{\gamma})
    \end{align*}
    for some $0 < \gamma \le (1-(2p+q))/2$ and 
    $$\frac{\mathbb{E}_1W_m(x)}{\sqrt{(2 + \delta)\log\log m}} \to \infty.$$
    Meanwhile, 
    $$\Var_1W_m(x) = \frac{F_m(x)(1 - F_m(x))}{x(1-x)} = O(1).$$
    Hence, as $HC_m^* \ge W_m(x)$ for all $x \in [0,1]$, by Chebyshev's inequality, for large enough constant $C, C'$,
    \begin{align*}
        \mathbb{P}_1(HC_m^* \le \sqrt{(2 + \delta)\log\log m}) &\le \mathbb{P}_1(W_m(x) \le \sqrt{(2 + \delta)\log\log m})\\
        &\le C\frac{\Var_1W_m(x)}{[\mathbb{E}_1W_m(x)]^2} \le C'n^{-2\gamma} \to 0
    \end{align*}
    which completes the proof.
\end{proof}

\subsection{Extension to the Soft Green/Red List Watermark}
When the generated text is expected to be very long, we can adapt our proposed scheme to the soft watermark; this may further mitigate the possible distortion during the text generation process when pseudorandomness based on the previous tokens is used. To do this, we simply replace $Q_{w}$ with the method described in Equation \eqref{eq:prompt green list}. In this case, we will not zero out the probability of sampling a token from the red list $\cR$ in $Q_{w}$. Thus, even in the unfortunate situation where the last $k$ tokens result in a very low pseudorandom number, there is still a slight possibility of choosing a word from the ``red list.'' However, this flexibility can diminish the method's ability to detect the watermark: firstly, we can rewrite Lemma \ref{lem:conditioned_on_green} as follows:
\begin{lemma}
    \label{lem:conditioned_on_green_soft}
    Let $\bm P_t$ be the original token distribution at step $t$ and $w$ be the next token sampled from the above decoding scheme. Then the conditional distribution of $\zeta_t$, given $w \in \cG_t$, is uniformly distributed over 
    $$
    \left[0, \frac{1 + (\mathrm{e}^{\delta} - 1)P_{t, G_t}}{\mathrm{e}^{\delta}} = P_{t, G_t} + \frac{1 - P_{t, G_t}}{\mathrm{e}^{\delta}}\right].
    $$
\end{lemma}

This revised statement suggests that when a green token is chosen, the watermark signal can be much weaker than in the original scheme. Secondly, the symmetry observed in Remark \ref{rem: red list} should be revised accordingly. If the next token $w$ is red, then the conditional distribution of $\zeta_t$ can be either a standard uniform random variable (if the first sample $w$ is not rejected) or a uniform random variable in the interval $\left[P_{t,\cG_t} + \frac{1 - P_{t,\cG_t}}{\mathrm{e}^{\delta}}, 1\right]$ (if the first sample $w$ is rejected). Hence, the proportion of signals coming from the red tokens will be much smaller than in the original scheme. Nonetheless, it is evident that analogous asymptotic guarantees to Proposition \ref{prop:powerful_detection_sum} and Theorem \ref{thm:powerful_detection_HC} can be established with appropriate assumptions.

\subsection{Speculative Decoding}
\label{sec: speculative details}

Speculative decoding \citep{leviathan2023fast, chen2023accelerating} is an algorithm originally designed to speed up sampling text from a large target language model by using a smaller, faster model, and has been widely adopted in production. Essentially, for a lookahead parameter $L$, the small draft model generates $L$ tokens ahead at each step, which the larger target model then evaluates in parallel, either accepting the proposed sequence up to the first rejected token or falling back to standard autoregressive generation if the first token is rejected. This process is formalized in Algorithm \ref{alg: speculative sampling}.

\begin{algorithm}
\caption{Speculative Sampling with Draft Model}
\label{alg: speculative sampling}
\begin{algorithmic}
\STATE \textbf{Input:} Target model $\cM_P$, draft model $\cM_Q$, random variables $\zeta_1,...,\zeta_L \in [0,1]$, prompt $x$
\STATE \textbf{Output:} Accepted tokens sequence $y$
\FOR{$t=1$ to $L$}
\STATE Compute $\bm Q_t = \cM_P(\cdot|x, w_1, ..., w_t-1)$ and sample $w_t$ from $\bm Q_t$
\ENDFOR
\STATE Compute $\bm P_t = \cM_P(\cdot|x, w_1, ..., w_t)$ for $t=1, ..., L$ in parallel
\FOR{$t=1$ to $L$}
\IF{$\zeta_t \cdot Q_{t, w} > P_{t, w}$}
\STATE Sample $w'_t$ from the normalized excess distribution $\varpropto \max(0, \bm P_t - \bm Q_t)$
\STATE \textbf{return} $(w_{1:t-1}, w'_t)$
\ENDIF
\ENDFOR
\STATE \textbf{return} $(w_{1:L})$
\end{algorithmic}
\end{algorithm}

\subsubsection{Connections to Maximal Coupling} When the lookahead parameter $L$ equals 1, the algorithm simplifies to a basic rejection sampling scheme, as shown in Algorithm \ref{alg: rejection sampling}. Now we prove that this process is equivalent to the maximal coupling defined in Algorithm \ref{alg: token_sampling_coupling} if $\zeta$ is a standard uniform random variable. Let $\texttt{Accept}$ be the event where we keep the initially sampled token. If we resample $w$ from the normalized excess distribution $\varpropto \max(0, \bm P_t - \bm Q_t)$, the process is equivalent to that of the maximal coupling in Algorithm \ref{alg: token_sampling_coupling}. Hence, we first see that $w | \texttt{Accept}$ follows the normalized overlap distribution $\sim \min(\bm P, \bm Q)$: with a slight abuse of notation, we use $\bP[w]$ to denote the probability of sampling the token $w \in \cW$ by following Algorithm \ref{alg: token_sampling_coupling}; then 
$$
\bP[w | \texttt{Accept}] = Q_w \cdot \bP[\zeta \le \min(1, P_w/Q_w)] = Q_w \cdot \min(1, P_w/Q_w) =  \min(P_w, Q_w)
$$
Now, it remains to show that $\bP[\texttt{Accept}] = \sum_{w \in \cW}\min(P_w, Q_w)$:
$$
    \bP[\texttt{Accept}] = \bP_{w \sim \bm Q}[\zeta \cdot Q_w > P_w]= \sum_{w \in \cW} Q_w \cdot \min(\frac{P_w}{Q_w}, 1)= \sum_{w \in \cW} \min(P_w, Q_w) = p
$$
        
\begin{algorithm}
    \caption{One Step Speculative Sampling for Next Token Generation}
    \label{alg: rejection sampling}
    \begin{algorithmic}
    \STATE \textbf{Input:} Token distributions $\bm P, \bm Q$, random variable $\zeta \in [0,1]$
    \STATE \textbf{Output:} Sampled token $w$
    \STATE Independently sample $w \sim \bm Q$
    \IF{$\zeta \cdot Q_w > P_w$}
        \STATE Sample $w'$ from the normalized excess distribution $\varpropto \max(0, \bm P - \bm Q)$
        \STATE $w \leftarrow w'$
    \ENDIF
    \STATE \textbf{return} $w$
    \end{algorithmic}
\end{algorithm}

\subsubsection{Speculative Decoding as Post-Processing}

To investigate the performance of the watermarking schemes under targeted modifications, we consider the following speculative decoding setup. We view the NTP generated by a watermarked decoder as a probability vector $\bm Q$ conditioned on the previous context. This context influences both the original language model's NTP generation and the pseudorandom variable $\tilde \zeta$ that determines the watermark signal. Concretely, let $\tilde{\bm P}$ denote the NTP generated by the original draft language model. For the Gumbel-max watermarking scheme, $\bm Q$ is given by
$$Q_w = \mathbb{1}_{\{ w = \argmax_{w' \in \cW} \log U_{w'}/\tilde P_w'\}}$$
where $U_{w'}$ is the pseudorandom variables determined by the prevoius tokens. For our proposed watermarking scheme,  $\bm Q$ is given by
$$Q_w = \mathbb{1}_{\{ \tilde \zeta \le \tilde P_{\cG_t} \}}\frac{\tilde P_w}{\tilde P_{\cG_t}} + \mathbb{1}_{\{\tilde \zeta > \tilde P_{\cG_t} \}}\frac{\tilde P_w}{1 - \tilde P_{\cG_t}}$$
where $\tilde P_{\cG_t} = \sum_{w' \in \cG_t} \tilde P_{w'}$, $\cG_t$ is the green list determined by the previous tokens and $\tilde \zeta$ is the pseudorandom variable used for watermarking. 

We assume the target model produces an NTP $\bm P$ that, given the previous tokens, is independent of the pseudorandom variables used for watermarking the draft model. This means the target model makes judgments based solely on $\bm P$, without knowledge of the watermark signal. Simply applying Algorithm \ref{alg: rejection sampling} to $\bm P, \bm Q$ and an independent standard uniform variable $\zeta$ would erase the watermark signal, as the marginal distribution of the sampled token $w$ would match the original NTP $\bm P$. Therefore, we model the target model as a ``lazy'' editor by modifying the rejection sampling condition to $0.5 \cdot \zeta \cdot Q_w > P_w$. This means the target model only accepts the draft model's suggestion when it has reasonable confidence in generating the same token. For instance, if $Q_w = 1$, the target model accepts the suggestion only if it has at least 50\% confidence in generating that token. This assumption simulates a human editor who prefers making minimal modifications to the draft model's suggestions.

\subsection{More Experimental Results}

\subsubsection{Additional Simulation Studies}

Complementary to Figure \ref{fig:small_m}, Supplementary Figure \ref{fig:small_m_add} also shows that the sum test tends to outperform the higher criticism when m is relatively small. 
\begin{figure}[hbt!]
    \captionsetup{name=Supplementary Table}
    \centering
    \includegraphics[width=0.99\textwidth]{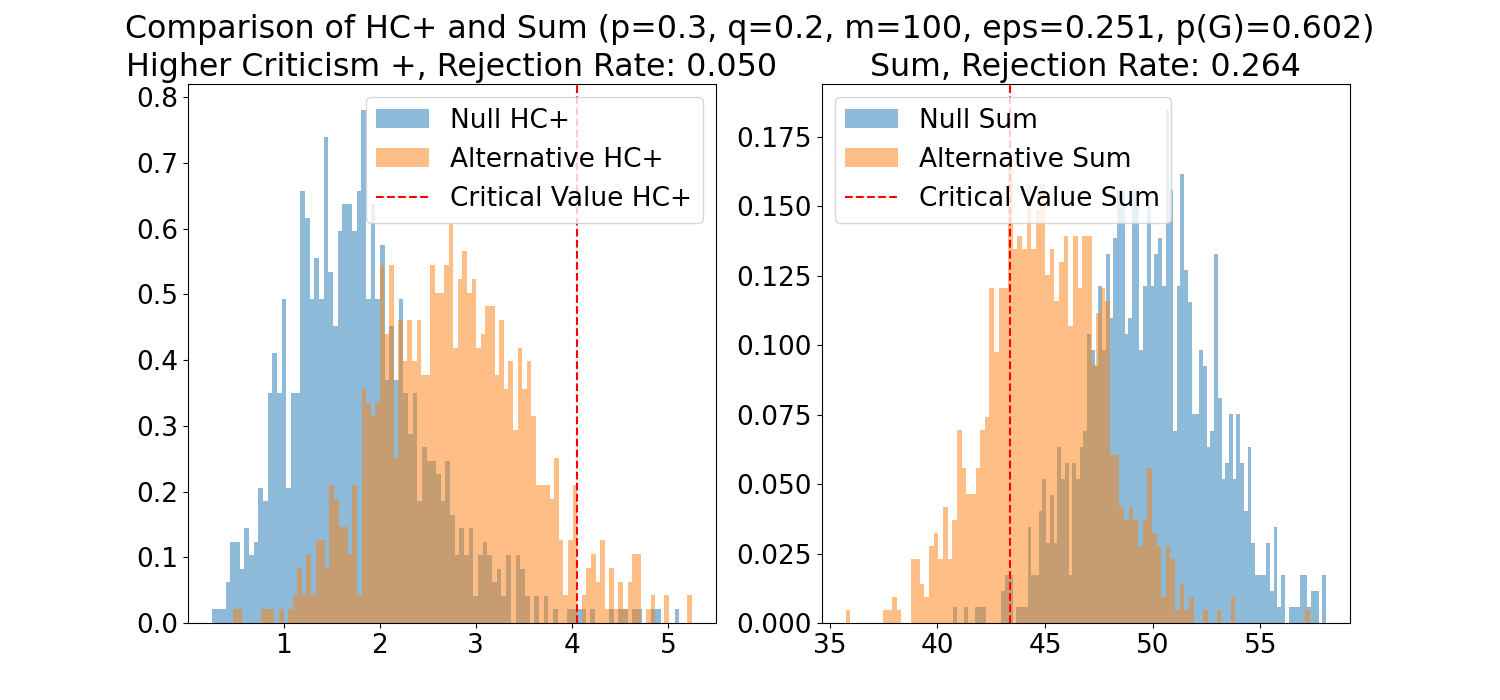}
    \caption{The histograms of the test statistics under the null and alternative hypotheses for the sum test and the higher criticism when $m = 100$. The test statistics are computed from the simulation under both the null and alternative hypotheses 2000 times. The critical value is determined by the quantile of the simulated histogram of the test statistic under the null hypothesis. The rejection rate, or the power, of each test, is computed by computing, under the alternative hypothesis, the proportion of the test statistics that are on the other side of the critical value. See also Figure \ref{fig:small_m}.}
    \label{fig:small_m_add}
\end{figure}

\subsubsection{Further Implementation Details on Language Model Experiments}
\label{sec: further exp details}
In our experiments, we prompt two family of instruction fine-tuned models with their respective default instruction prompting templates: Meta's \texttt{LLaMA} model family \citep{dubey2024llama} and Microsoft's \texttt{Phi-3} family \citep{abdin2024phi}. For each experiment, we embed watermarks to \texttt{Llama-3.2-1B-Instruct} (1B parameters) and \texttt{Phi-3-mini-4k-instruct} (3.8B parameters). For both models, we apply Flash Attention to optimize for faster inference~\citep{saha2024complexity}. For the two datasets, we use the same 200 samples from both datasets as used in \citep{tu2023waterbench}. These sub-samples of the two datasets provide short input instructions but should yield relatively long answers: for ELI5, the average length of the input question is 41.04 words while the average length of the reference answer is 236.6 words; similarly, for FinQA, the average length of the input question is 13.67 while the average length of the reference answer is 251.13 words. 
Our prompts follow the chat template of each respective model. 

For the speculative decoding setup, we choose \texttt{Llama-3.2-3B-Instruct} (3B parameters) as the target model for the \texttt{LLaMA} family, and \texttt{Phi-3-medium-4k-instruct} (14 B parameters) for the \texttt{Phi-3} family.

For our sum test statistic, due to precision errors and computational complexity in calculating the exact Irwin-Hall distribution, when computing the $p$-values, we use normal approximation when $n \ge 15$. Unless otherwise specified, we flag a text as watermarked if the $p$-value is lower than $\alpha = 0.01$, ensuring a false positive rate to be less than $0.01$. For higher criticism, instead of computing the $p$-values, we precompute the quantile of the simulated histogram of the test statistic under the null hypothesis to determine the empirical critical value at the significance level $\alpha = 0.01$. 

\subsubsection{Ablation Study on the Parameters}
\label{sec: ablation study}

\begin{table*}[t]
\captionsetup{font={stretch=1}, name=Supplementary Table}
\renewcommand{\arraystretch}{0.8}
\caption{Comparison of the S-BERT similarity scores and the true positive rates (TPR) among four watermarking schemes: (1) the Gumbel-max watermark \citep{aaronson}, (2) the green/red list watermark \citep{kirchenbauer2023watermark}, (3) the DiPmark method \cite{wu2023dipmark}, and (4) our proposed one.}
\label{tab: S-BERT and TPR, add}
\begin{small}
\centering
\begin{tabular}{c c c l l | r | r | r | r}
\toprule
Model & Data & $k$ & Metric & & Gumbel-max & Green/red list & DiPmark & Ours \\
\midrule
\multirow{20}{*}{\rotatebox[origin=c]{90}{Llama-3.2-1B-Instruct}} & \multirow{10}{*}{\rotatebox[origin=c]{90}{FinQA}} & \multirow{5}{*}{2} & \multirow{2}{*}{S-BERT} & mean & 0.7945 & 0.8038 & 0.8098 & 0.8136 \\
 &  &  &  & median & 0.8175 & 0.8292 & 0.8333 & 0.8347\\
 &  &  & TPR &  & 1.0000 & 0.7550 & 0.9200 & \textbf{0.9500}\\
 &  &  & TPR aug. &  & 0.9950 & 0.5150 & 0.7350 & \textbf{0.8550}\\
 &  &  & TPR para. &  & 0.9800 & 0.4100 & 0.5400 & \textbf{0.6800}\\
\cmidrule{3-9}
 &  & \multirow{5}{*}{4} & \multirow{2}{*}{S-BERT} & mean & 0.8020 & 0.8270 & 0.8086 & 0.8117 \\
 &  &  &  & median & 0.8344 & 0.8458 & 0.8328 & 0.8426\\
 &  &  & TPR &  & 0.9850 & 0.8300 & 0.9700 & \textbf{0.9750}\\
 &  &  & TPR aug. &  & 0.9850 & 0.4500 & 0.7550 & \textbf{0.8400}\\
 &  &  & TPR para. &  & 0.9500 & 0.2650 & 0.4300 & \textbf{0.6450}\\
\cmidrule{2-9}
 & \multirow{10}{*}{\rotatebox[origin=c]{90}{ELI5}} & \multirow{5}{*}{2} & \multirow{2}{*}{S-BERT} & mean & 0.7171 & 0.7195 & 0.7250 & 0.7351 \\
 &  &  &  & median & 0.7342 & 0.7315 & 0.7393 & 0.7429\\
 &  &  & TPR &  & 1.0000 & 0.9100 & 0.9950 & \textbf{1.0000}\\
 &  &  & TPR aug. &  & 1.0000 & 0.6650 & 0.9100 & \textbf{1.0000}\\
 &  &  & TPR para. &  & 1.0000 & 0.4500 & 0.8600 & \textbf{0.9200}\\
\cmidrule{3-9}
 &  & \multirow{5}{*}{4} & \multirow{2}{*}{S-BERT} & mean & 0.7302 & 0.7226 & 0.7250 & 0.7315 \\
 &  &  &  & median & 0.7447 & 0.7433 & 0.7346 & 0.7433\\
 &  &  & TPR &  & 1.0000 & 0.9700 & \textbf{1.0000} & \textbf{1.0000}\\
 &  &  & TPR aug. &  & 1.0000 & 0.6450 & 0.9350 & \textbf{0.9900}\\
 &  &  & TPR para. &  & 1.0000 & 0.4000 & 0.7800 & \textbf{0.9200}\\
\bottomrule
\end{tabular}
\end{small}
\end{table*}

Supplementary Table \ref{tab: S-BERT and TPR, add} present the analogous results to Table \ref{tab: S-BERT and TPR, alpha 0.01, main} for the \texttt{Llama-3.2-1B-Instruct} model. The results are consistent with the main results, showing that our proposed watermarking scheme outperforms the Gumbel-max method in terms of the text distortion, and more competitive TPR than the green/red list and DiPMark methods.

\begin{table}
\captionsetup{font={stretch=1}, name=Supplementary Table}
\renewcommand{\arraystretch}{0.8}
\caption{Comparison of the S-BERT similarity scores and the true positive rates (TPR) of the two watermarking schemes: green/red list \cite{kirchenbauer2023watermark} and our proposed scheme for the Phi-3-mini-4k-instruct (3.8B). }
\label{table:phi}
\centering
\begin{tabular}{l l l l | r r | r r | r}
\toprule
\multirow{3}{*}{\rotatebox[origin=c]{90}{Data}} & \multirow{3}{*}{k} & \multirow{3}{*}{Metric} & & \multicolumn{4}{c}{Green/red list} & \multirow{2}{*}{Our method} \\
&&&& \multicolumn{2}{c}{$\delta = 2$} & \multicolumn{2}{c}{$\delta = 1$} & \\
&&&& $\gamma = 0.25$ & $\gamma = 0.5$ & $\gamma = 0.25$ & $\gamma = 0.5$ & $\gamma = 0.5$\\
\midrule
\multirow{8}{*}{\rotatebox[origin=c]{90}{FinQA}} & \multirow{4}{*}{2} & \multirow{2}{*}{S-BERT} & mean & 0.7887 & 0.8001 & 0.8269 & 0.8167 & 0.7963 \\
 &  &  & median & 0.8229 & 0.8404 & 0.8598 & 0.8472 & 0.8382 \\
 &  & TPR &  & 0.9700 & 0.9700 & 0.5900 & 0.6550 & 0.9750 \\
 &  & TPR aug. &  & 0.9150 & 0.8850 & 0.3250 & 0.3500 & 0.8700 \\
\cmidrule{3-9}
 & \multirow{4}{*}{4} & \multirow{2}{*}{S-BERT} & mean & 0.7957 & 0.7967 & 0.8305 & 0.8233 & 0.8023 \\
 &  &  & median & 0.8241 & 0.8277 & 0.8528 & 0.8494 & 0.8288 \\
 &  & TPR &  & 0.9650 & 0.9600 & 0.7850 & 0.7750 & 0.9500 \\
 &  & TPR aug. &  & 0.8750 & 0.8900 & 0.4450 & 0.5150 & 0.8300 \\
\cmidrule{2-9}
\multirow{8}{*}{\rotatebox[origin=c]{90}{ELI5}} & \multirow{4}{*}{2} & \multirow{2}{*}{S-BERT} & mean & 0.7022 & 0.6962 & 0.7084 & 0.7178 & 0.7161 \\
 &  &  & median & 0.7175 & 0.7060 & 0.7264 & 0.7239 & 0.7290 \\
 &  & TPR &  & 1.0000 & 1.0000 & 0.7800 & 0.9100 & 0.9900 \\
 &  & TPR aug. &  & 0.9900 & 0.9850 & 0.4900 & 0.6500 & 0.9350 \\
\cmidrule{3-9}
 & \multirow{4}{*}{4} & \multirow{2}{*}{S-BERT} & mean & 0.7030 & 0.7161 & 0.7060 & 0.7084 & 0.7185 \\
 &  &  & median & 0.7076 & 0.7299 & 0.7214 & 0.7236 & 0.7344 \\
 &  & TPR &  & 1.0000 & 0.9950 & 0.8850 & 0.8850 & 1.0000 \\
 &  & TPR aug. &  & 0.9500 & 0.9850 & 0.5450 & 0.5750 & 0.9550 \\
\bottomrule
\end{tabular}
\end{table}

\begin{table}
\captionsetup{font={stretch=1}, name=Supplementary Table}
\renewcommand{\arraystretch}{0.8}
\caption{Same as Supplementary Table \ref{table:phi}, but for the Llama model.}
\label{table:llama}
\centering
\begin{tabular}{l l l l | r r | r r | r}
\toprule
\multirow{3}{*}{\rotatebox[origin=c]{90}{Data}} & \multirow{3}{*}{k} & \multirow{3}{*}{Metric} & & \multicolumn{4}{c}{Green/red list} & \multirow{2}{*}{Our method} \\
&&&& \multicolumn{2}{c}{$\delta = 2$} & \multicolumn{2}{c}{$\delta = 1$} & \\
&&&& $\gamma = 0.25$ & $\gamma = 0.5$ & $\gamma = 0.25$ & $\gamma = 0.5$ & $\gamma = 0.5$\\
\midrule
\multirow{8}{*}{\rotatebox[origin=c]{90}{FinQA}} & \multirow{4}{*}{2} & \multirow{2}{*}{S-BERT} & mean & 0.7962 & 0.8025 & 0.8199 & 0.8038 & 0.8136 \\
 &  &  & median & 0.8144 & 0.8290 & 0.8391 & 0.8292 & 0.8347 \\
 &  & TPR &  & 0.9900 & 0.9950 & 0.6750 & 0.7550 & 0.9500 \\
 &  & TPR aug. &  & 0.9350 & 0.9550 & 0.4450 & 0.5150 & 0.8550 \\
\cmidrule{3-9}
 & \multirow{4}{*}{4} & \multirow{2}{*}{S-BERT} & mean & 0.8059 & 0.7940 & 0.8247 & 0.8270 & 0.8117 \\
 &  &  & median & 0.8318 & 0.8127 & 0.8381 & 0.8458 & 0.8426 \\
 &  & TPR &  & 0.9750 & 0.9850 & 0.8300 & 0.8300 & 0.9750 \\
 &  & TPR aug. &  & 0.9250 & 0.9300 & 0.5150 & 0.4500 & 0.8400 \\
\cmidrule{2-9}
\multirow{8}{*}{\rotatebox[origin=c]{90}{ELI5}} & \multirow{4}{*}{2} & \multirow{2}{*}{S-BERT} & mean & 0.7077 & 0.7040 & 0.7196 & 0.7195 & 0.7351 \\
 &  &  & median & 0.7181 & 0.7267 & 0.7382 & 0.7315 & 0.7429 \\
 &  & TPR &  & 1.0000 & 1.0000 & 0.9250 & 0.9100 & 1.0000 \\
 &  & TPR aug. &  & 0.9950 & 0.9900 & 0.7150 & 0.6650 & 0.9650 \\
\cmidrule{3-9}
 & \multirow{4}{*}{4} & \multirow{2}{*}{S-BERT} & mean & 0.7169 & 0.7008 & 0.7279 & 0.7226 & 0.7315 \\
 &  &  & median & 0.7385 & 0.7243 & 0.7514 & 0.7433 & 0.7433 \\
 &  & TPR &  & 1.0000 & 1.0000 & 0.9400 & 0.9700 & 1.0000 \\
 &  & TPR aug. &  & 1.0000 & 0.9950 & 0.6450 & 0.6450 & 0.9900 \\
\bottomrule
\end{tabular}
\end{table}

Supplementary Table \ref{table:phi} and Supplementary Table \ref{table:llama} present the analogous results with a wider range of parameters. Supplementary Table \ref{table:llama} contains the results for the Llama model, while Supplementary Table \ref{table:phi} contains the results for the Phi model. From both tables, we observe a positive correlation between the distortion of the generation (lower S-BERT scores) and the detection power (higher TPR) for the green/red list method, echoing existing theory and empirical findings \citep{kirchenbauer2023watermark, fernandez2023three, piet2023mark}. 

\begin{table}[t]
\captionsetup{font={stretch=1}, name=Supplementary Table}
\renewcommand{\arraystretch}{0.8}
\caption{Same as Table \ref{tab:token_stats}, but for the Llama model.}
\label{tab:token_stats, add}
\centering
\begin{small}
\begin{tabular}{l l l l | r}
\toprule
Model & Dataset & k & Method & Repeated (\%) \\
\midrule
\multirow{16}{*}{\rotatebox[origin=c]{90}{Llama-3.2-1B-Instruct}} & \multirow{8}{*}{FinQA} & \multirow{4}{*}{2} & Gumbel-max & 25.16\% \\
 &  &  & Green/red list & 19.39\% \\
 &  &  & DiPMark & 21.03\% \\
 &  &  & Ours & 21.29\% \\
\cmidrule{3-5}
 &  & \multirow{4}{*}{4} & Gumbel-max & 7.81\% \\
 &  &  & Green/red list & 5.94\% \\
 &  &  & DiPMark & 5.85\% \\
 &  &  & Ours & 6.17\% \\
\cmidrule{2-5}
 & \multirow{8}{*}{ELI5} & \multirow{4}{*}{2} & Gumbel-max & 25.25\% \\
 &  &  & Green/red list & 18.81\% \\
 &  &  & DiPMark & 19.58\% \\
 &  &  & Ours & 19.50\% \\
\cmidrule{3-5}
 &  & \multirow{4}{*}{4} & Gumbel-max & 5.62\% \\
 &  &  & Green/red list & 3.68\% \\
 &  &  & DiPMark & 4.15\% \\
 &  &  & Ours & 3.84\% \\
\bottomrule
\end{tabular}
\end{small}
\end{table}

Supplementary Table \ref{tab:token_stats, add} presents the analogous results to Table \ref{tab:token_stats} for the Llama model. Here, we observe that for the Llama model, the Gumbel-max method has a higher rate of repeated tokens compared to all other methods.

\begin{table*}[hbt!]
    \captionsetup{font={stretch=1}, name=Supplementary Table}
    \renewcommand{\arraystretch}{0.8}
    \caption{Same as Table \ref{tab:speculative_stats}, but for the Llama model.}
    \label{tab:speculative_stats, add}
    \centering
    \begin{small}
    \centering
    \begin{tabular}{l l l l | r | r}
    \toprule
    Model & Data & $k$ & Metric & Gumbel-max & Ours \\
    \midrule
    \multirow{8}{*}{Llama-3.2-1B-Instruct} & \multirow{4}{*}{FinQA} & \multirow{2}{*}{2} & Avg. Rejection Rate & 22.28\% & 16.00\% \\
     &  &  & TPR & 0.2300 & 0.1900 \\
    \cmidrule{3-6}
     &  & \multirow{2}{*}{4} & Avg. Rejection Rate & 22.83\% & 16.85\% \\
     &  &  & TPR & 0.2100 & 0.2000 \\
    \cmidrule{2-6}
     & \multirow{4}{*}{ELI5} & \multirow{2}{*}{2} & Avg. Rejection Rate & 26.69\% & 18.42\% \\
     &  &  & TPR & 0.3650 & 0.2800 \\
    \cmidrule{3-6}
     &  & \multirow{2}{*}{4} & Avg. Rejection Rate & 28.23\% & 19.42\% \\
     &  &  & TPR & 0.4100 & 0.3600 \\
    \bottomrule
    \end{tabular}
    \end{small}
\end{table*} 

Supplementary Table \ref{tab:speculative_stats, add} presents the analogous results to Table \ref{tab:speculative_stats} for the Llama model. The results show that our proposed watermarking scheme has a lower average rejection rate compared to the Gumbel-max method, and this is consistent with the main results.

\begin{table*}[hbt!]
\captionsetup{font={stretch=1}, name=Supplementary Table}
\renewcommand{\arraystretch}{0.8}
\caption{Same as Table \ref{tab:stats}, but for the Llama model.}
\label{tab:stats, add}
\centering
\begin{small}
\centering
\begin{tabular}{c c c l l | r | r | r}
\toprule
Model & Data & $k$/Num. Tokens & Metric & Max & Sum & HC$^+$ \\
\midrule
\multirow{8}{*}{Llama-3.2-1B-Instruct} & \multirow{4}{*}{FinQA} & \multirow{2}{*}{2/382.170} & TPR & 0.335 & \textbf{0.95} & 0.905 \\
 &  &  & TPR aug. & 0.125 & \textbf{0.855} & 0.665 \\
\cmidrule{3-7}
 &  & \multirow{2}{*}{4/411.095} & TPR & 0.26 & \textbf{0.975} & 0.93 \\
 &  &  & TPR aug. & 0.075 & \textbf{0.84} & 0.615 \\
\cmidrule{2-7}
 & \multirow{4}{*}{ELI5} & \multirow{2}{*}{2/383.360} & TPR & 0.4 & \textbf{1} & 0.99 \\
 &  &  & TPR aug. & 0.165 & \textbf{0.965} & 0.87 \\
\cmidrule{3-7}
 &  & \multirow{2}{*}{4/414.440} & TPR & 0.385 & \textbf{1} & \textbf{1} \\
 &  &  & TPR aug. & 0.12 & \textbf{0.99} & 0.905 \\
\bottomrule
\end{tabular}
\end{small}
\end{table*}

Supplementary Table \ref{tab:stats, add} presents the analogous results to Table \ref{tab:stats} for the Llama model. The results show higher true positive rate for the sum test compared to the higher criticism when the number of tokens used to detect the watermark is relatively small, and this consistent with the main results.

\begin{table*}[hbt!]
\captionsetup{font={stretch=1}, name=Supplementary Table}
\renewcommand{\arraystretch}{0.8}
\caption{Comparison of the watermarking scheme with a single list and the watermarking scheme with disparate lists. The true positive rate (TPR) and the true positive rate after the substitution attack (TPR aug.) are reported in the form of (Sum statistic/HC$^+$ statistic).}
\label{tab: one list}
\centering
\begin{small}
\begin{tabular}{l l l l l | r r}
\toprule
\multirow{2}{*}{Model} & \multirow{2}{*}{Data} & \multirow{2}{*}{k} & \multirow{2}{*}{Metric} & & \multicolumn{2}{c}{Our Method} \\
&&&&& single list & disparate lists\\
\midrule
\multirow{16}{*}{\rotatebox[origin=c]{90}{Phi-3-mini-4k-instruct (3.8B)}} & \multirow{8}{*}{\rotatebox[origin=c]{90}{FinQA}} & \multirow{4}{*}{2} & \multirow{2}{*}{S-BERT} & mean & \textbf{0.8121} & 0.7963 \\
 &  &  &  & median & \textbf{0.8431} & 0.8382 \\
 &  &  & TPR &  & 0.955/0.885 & \textbf{0.975}/\textbf{0.92} \\
 &  &  & TPR aug. &  & \textbf{0.905}/\textbf{0.73} & 0.87/0.71 \\
\cmidrule{3-7}
 &  & \multirow{4}{*}{4} & \multirow{2}{*}{S-BERT} & mean & \textbf{0.808} & 0.8023 \\
 &  &  &  & median & \textbf{0.8494} & 0.8288 \\
 &  &  & TPR &  & \textbf{0.955}/0.89 & 0.95/\textbf{0.91} \\
 &  &  & TPR aug. &  & \textbf{0.84}/\textbf{0.665} & 0.83/0.625 \\
\cmidrule{2-7}
 & \multirow{8}{*}{\rotatebox[origin=c]{90}{ELI5}} & \multirow{4}{*}{2} & \multirow{2}{*}{S-BERT} & mean & \textbf{0.7224} & 0.7161 \\
 &  &  &  & median & \textbf{0.7499} & 0.729 \\
 &  &  & TPR &  & \textbf{0.995}/\textbf{0.99} & 0.99/0.98 \\
 &  &  & TPR aug. &  & \textbf{0.96}/\textbf{0.895} & 0.935/0.865 \\
\cmidrule{3-7}
 &  & \multirow{4}{*}{4} & \multirow{2}{*}{S-BERT} & mean & \textbf{0.7237} & 0.7185 \\
 &  &  &  & median & \textbf{0.7409} & 0.7344 \\
 &  &  & TPR &  & \textbf{1}/\textbf{0.995} & \textbf{1}/0.99 \\
 &  &  & TPR aug. &  & \textbf{0.965}/0.82 & 0.955/\textbf{0.83} \\
\midrule
\multirow{16}{*}{\rotatebox[origin=c]{90}{Llama-3.2-1B-Instruct}} & \multirow{8}{*}{\rotatebox[origin=c]{90}{FinQA}} & \multirow{4}{*}{2} & \multirow{2}{*}{S-BERT} & mean & 0.8013 & \textbf{0.8136} \\
 &  &  &  & median & 0.8314 & \textbf{0.8347} \\
 &  &  & TPR &  & \textbf{0.98}/\textbf{0.94} & 0.95/0.905 \\
 &  &  & TPR aug. &  & \textbf{0.9}/\textbf{0.74} & 0.855/0.665 \\
\cmidrule{3-7}
 &  & \multirow{4}{*}{4} & \multirow{2}{*}{S-BERT} & mean & \textbf{0.8128} & 0.8117 \\
 &  &  &  & median & 0.8402 & \textbf{0.8426} \\
 &  &  & TPR &  & 0.97/0.925 & \textbf{0.975}/\textbf{0.93} \\
 &  &  & TPR aug. &  & \textbf{0.86}/\textbf{0.645} & 0.84/0.615 \\
\cmidrule{2-7}
 & \multirow{8}{*}{\rotatebox[origin=c]{90}{ELI5}} & \multirow{4}{*}{2} & \multirow{2}{*}{S-BERT} & mean & 0.7287 & \textbf{0.7351} \\
 &  &  &  & median & 0.7355 & \textbf{0.7429} \\
 &  &  & TPR &  & \textbf{1}/\textbf{0.99} & \textbf{1}/\textbf{0.99} \\
 &  &  & TPR aug. &  & 0.96/0.85 & \textbf{0.965}/\textbf{0.87} \\
\cmidrule{3-7}
 &  & \multirow{4}{*}{4} & \multirow{2}{*}{S-BERT} & mean & 0.7243 & \textbf{0.7315} \\
 &  &  &  & median & \textbf{0.7453} & 0.7433 \\
 &  &  & TPR &  & \textbf{1}/0.99 & \textbf{1}/\textbf{1} \\
 &  &  & TPR aug. &  & 0.985/0.835 & \textbf{0.99}/\textbf{0.905} \\
\bottomrule
\end{tabular}
\end{small}
\end{table*}

\subsubsection{A Single Green List}
\label{app:reuse_green_list}

Notice that our theory requires that the green lists $\cG$ and the uniform random variables $\zeta$ be independent, and our scheme does not necessitate creating a new green list before generating each token. Instead, as long as the user does not have access to the watermark key, a priori, we can simply generate a single green list at the beginning and use it for generating all the tokens. This way, we only need to share a single pseudorandom function to generate random variable $\zeta$'s. We performed an ablation study on this in Supplementary Table \ref{tab: one list}. In particular, we compare the results of our watermarking scheme when we use a single green list and when we use disparate green lists. We report the BERT similarity scores and the TPR of the test statistic based on the sum of the $\zeta$'s and the higher criticism statistic (in the form TPR for sum/TPR for higher criticism). We conclude that having a single green list does not significantly impact our watermarking scheme when a larger context window is used to generate $\zeta$. An analogous observation was also made in \citet{zhao2023provable} for the green/red list soft watermarking scheme. However, it is unclear whether having the same green list for all the tokens will allow an adversary to have more power to learn the exact green list generator and produce attacks that can fool the detector, which is an interesting direction for future work.

\end{document}